\documentclass{article}

\usepackage{adjustbox}
\usepackage{longtable}

\usepackage{arxiv}

\usepackage{bm}
\usepackage{amssymb,amsthm,mathtools}

\DeclareMathOperator{\Arg}{Arg}
\DeclareMathOperator{\RePart}{Re}
\DeclareMathOperator{\ImPart}{Im}
\DeclareMathOperator{\argu}{arg}
\DeclareMathOperator{\cotan}{cot}

\newtheorem{theorem}{Theorem}
\newtheorem*{theorem*}{Theorem}
\newtheorem{lemma}{Lemma}
\newtheorem{claim}{Claim}
\newtheorem{conjecture}[theorem]{Conjecture}

\usepackage[utf8]{inputenc} 
\usepackage[T1]{fontenc}    
\usepackage{url}            
\usepackage{booktabs}       
\usepackage{amsfonts}       
\usepackage{nicefrac}       
\usepackage{microtype}      
\usepackage{lipsum}
\usepackage{multirow}
\usepackage{xcolor}
\usepackage{enumitem}
\usepackage{graphicx}
\usepackage{orcidlink}
\usepackage{comment}
\usepackage[most]{tcolorbox}
\usepackage{array}
\newcolumntype{R}[1]{>{\raggedleft\arraybackslash}p{#1}}
\newcolumntype{L}[1]{>{\raggedright\arraybackslash}p{#1}}

\usepackage[numbers]{natbib}
\usepackage{caption}
\usepackage{amsmath}
\usepackage{rotating}
\usepackage{cleveref}

\makeatletter
\renewcommand{\footnotesize}{\@setfontsize\footnotesize{8pt}{10pt}}
\makeatother

\makeatletter
\renewcommand{\scriptsize}{\@setfontsize\scriptsize{7pt}{9pt}}
\makeatother

\definecolor{VUB_blauw}{rgb}{0.1529, 0.2667, 0.5529}
\usepackage{hyperref}       
\hypersetup{
    colorlinks,%
    citecolor=VUB_blauw,%
    filecolor=VUB_blauw,%
    linkcolor=VUB_blauw,%
    urlcolor=VUB_blauw
}

\newcommand{\customCor}[1]{%
  \includegraphics[height=1em]{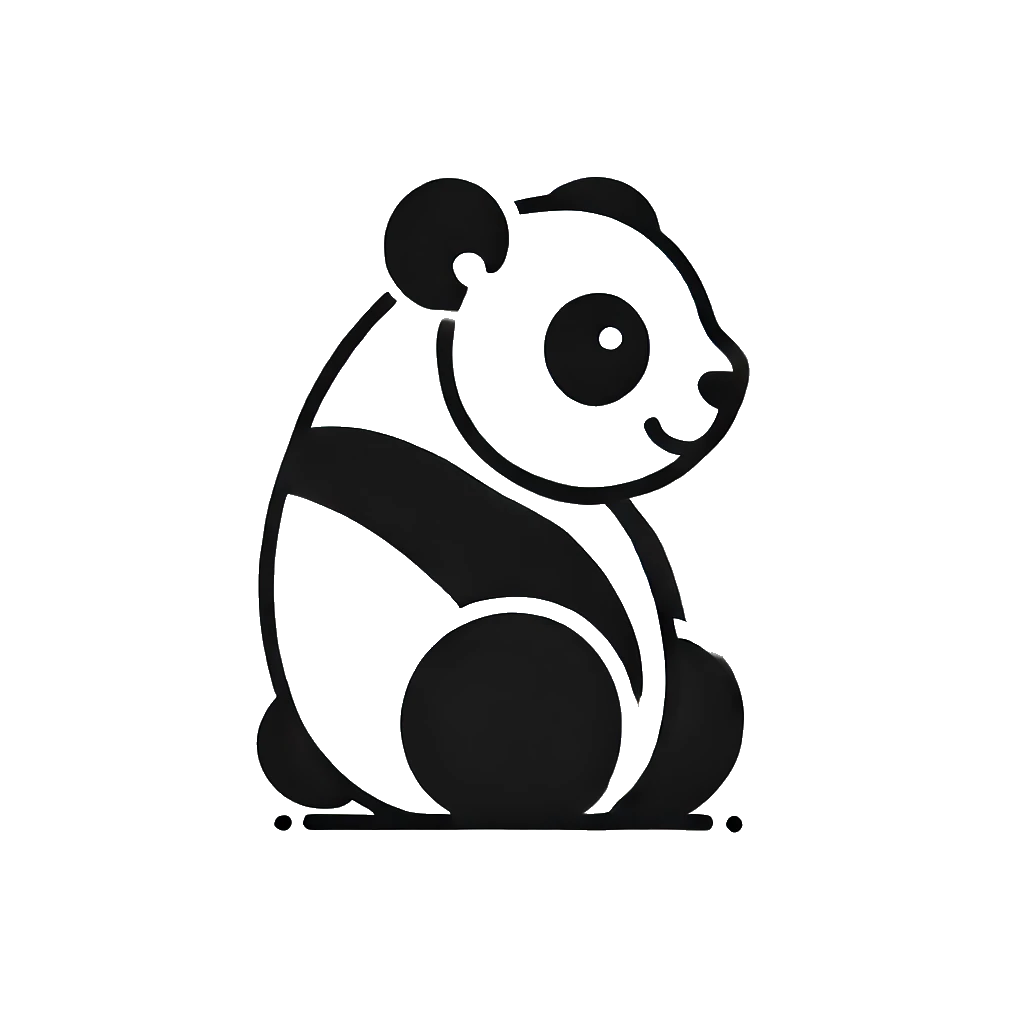} #1%
}

\AddToHook{shipout/background}{%
  \ifnum\value{page}=1 
    \pagenumbering{Roman} 
    \setcounter{page}{1} 
  \fi
  \ifnum\value{page}=2 
  \pagenumbering{arabic} 
  \setcounter{page}{2} 
  \fi
}

\newtcolorbox{versionzerobox}{
  enhanced,
  breakable,
  colback=red!2,
  colframe=red!40!black,
  boxrule=0.6pt,
  arc=2mm,
  left=2mm,right=2mm,top=1.5mm,bottom=1.5mm,
  title=\textbf{First version},
  fonttitle=\bfseries,
}

\newtcolorbox{versiononebox}{
  enhanced,
  breakable,
  colback=gray!3,
  colframe=black,
  boxrule=0.6pt,
  arc=2mm,
  left=2mm,right=2mm,top=1.5mm,bottom=1.5mm,
  title=\textbf{Second version},
  fonttitle=\bfseries,
}

\newtcolorbox{versiontwobox}{
  enhanced,
  breakable,
  colback=blue!2,
  colframe=blue!50!black,
  boxrule=0.6pt,
  arc=2mm,
  left=2mm,right=2mm,top=1.5mm,bottom=1.5mm,
  title=\textbf{Third version},
  fonttitle=\bfseries,
}

\newtcolorbox{versionthreebox}{
  enhanced,
  breakable,
  colback=green!2,
  colframe=green!40!black,
  boxrule=0.6pt,
  arc=2mm,
  left=2mm,right=2mm,top=1.5mm,bottom=1.5mm,
  title=\textbf{Final version},
  fonttitle=\bfseries,
}

\title{Early Evidence of Vibe-Proving with Consumer LLMs:\\
A Case Study on Spectral Region Characterization\\
with ChatGPT-5.2 (Thinking)}
\runningtitle{}

\author{
  Brecht Verbeken\textsuperscript{1,2,\customCor{ }}\\
  \orcidlinkc{0000-0002-7506-3298}\\
  \And
  Brando Vagenende\textsuperscript{1} \\ 
  \orcidlinkc{0000-0002-0573-8093} \\
  \And
  Marie-Anne Guerry\textsuperscript{1}\\
  \orcidlinkc{0000-0001-5842-8905}\\
  \AND
  Andres Algaba\textsuperscript{1,2} \\ 
  \orcidlinkc{0000-0002-0532-3066} \\
  \And
  Vincent Ginis\textsuperscript{1,2,3} \\ 
  \orcidlinkc{0000-0003-0063-9608} \\
  \and
  \textsuperscript{1}Data Analytics Lab, Vrije Universiteit Brussel, Pleinlaan 5, 1050 Brussel, Belgium \\ 
  \textsuperscript{2}imec-SMIT, Vrije Universiteit Brussel, Pleinlaan 9, 1050 Brussels, Belgium \\ 
  \textsuperscript{3}School of Engineering and Applied Sciences, Harvard University, Cambridge, Massachusetts 02138, USA
}

\begin{document}

\maketitle
\renewcommand{\thefootnote}{}
\footnotetext{\includegraphics[height=1em]{panda2.png} Corresponding author: \href{mailto:brecht.verbeken@vub.be}{brecht.verbeken@vub.be} \\}
\renewcommand{\thefootnote}{\arabic{footnote}}
\thispagestyle{plain}

\begin{abstract}
Large Language Models (LLMs) are increasingly used as scientific copilots, but evidence on their role in research-level mathematics remains limited, especially for workflows accessible to individual researchers. We present early evidence for vibe-proving with a consumer subscription LLM through an auditable case study that resolves Conjecture~20 of Ran and Teng (2024) on the exact nonreal spectral region of a $4$-cycle row-stochastic nonnegative matrix family. We analyze seven shareable ChatGPT-5.2 (Thinking) threads and four versioned proof drafts, documenting an iterative pipeline of generate, referee, and repair. The model is most useful for high-level proof search, while human experts remain essential for correctness-critical closure. The final theorem provides necessary and sufficient region conditions and explicit boundary attainment constructions. Beyond the mathematical result, we contribute a process-level characterization of where LLM assistance materially helps and where verification bottlenecks persist, with implications for evaluation of AI-assisted research workflows and for designing human-in-the-loop theorem proving systems.
\end{abstract}


\section{Introduction}

Vibe-coding, where programmers steer Large Language Models (LLMs) to generate software through high-level natural-language intent rather than line-by-line specification, has rapidly transformed code generation~\cite{sarkar2025vibecoding,peng2023impactcopilot,cui2024productivitycopilot}. Underpinning this shift is the sustained capability growth of frontier LLMs across successive systems and evaluations~\cite{bubeck2023sparks,dai2025pre,liu2025deepseek,singh2025openai}. As these models grow more capable, they are increasingly deployed as scientific collaborators: generating candidate research ideas in controlled studies~\cite{guevara2026single,si2024can}, acting as agentic research assistants that plan, search, and iterate over literature and experiments~\cite{baek2025researchagent,ghafarollahi2025sciagents,gottweis2025towards,lu2024ai,novikov2025alphaevolve,schmidgall2025agent,yamada2025ai}, and accelerating evidence synthesis and literature screening~\cite{delgado2025transforming,lehr2024chatgpt,skarlinski2024language}. At field scale, LLM tooling is associated with measurable shifts in scientific production and communication~\cite{kusumegi2025scientific}, but it remains unclear how reliably these systems support auditable, research-level mathematical workflows under consumer-access conditions~\cite{bubeck2025early,feng2026semiautonomousmathematicsdiscoverygemini}.

Beyond collaboration, LLMs increasingly execute end-to-end scientific workflows across domains. In the social sciences, models serve as simulated subjects and predictive surrogates~\cite{hewitt2024predicting,manning2024automated,ziems2024can}. In the natural sciences, agentic systems couple LLM reasoning to domain tools, autonomously running chemistry campaigns, designing biomolecules, and accelerating experimental protocols~\cite{boiko2023autonomous,m2024augmenting,zheng2025large,swanson2025virtual,rizvi2025scaling,smith2026labs,woodruff2026accelerating}, alongside work on test-time discovery and automated search~\cite{ttt-discover2026,vitvitskyi2026mining}. Mathematics provides an unusually crisp stress test for whether vibe-coding's paradigm can extend beyond software: correctness is, in principle, checkable, and performance on benchmark suites has improved rapidly~\cite{hendrycks2021measuring,ballon2026benchmarks,cobbe2021training,mirzadeh2024gsm,gao2024omni,glazer2024frontiermath,phan2025humanity}. Recent systems reach gold-medalist-level performance on Olympiad geometry~\cite{trinh2024solving,chervonyi2025gold}, and a gold-medal score on International Mathematical Olympiad (IMO) problems~\cite{castelvecchi2025ai}. At the research frontier, program search and large-scale exploration pipelines generate new mathematical artefacts~\cite{romera2024mathematical,swirszcz2025advancing,georgiev2025mathematical}, case studies document substantive human--LLM co-working on active problems~\cite{bubeck2025early,bryan2026motivic}, and community infrastructures curate open problem sets and evaluation tasks~\cite{optimization-constants-repo,unsolvedmath2026,erdos-problems}, including initiatives that explicitly invite LLM attempts~\cite{abouzaid2026proof}. Semi-autonomous pipelines have even been applied at scale to open mathematics problems~\cite{feng2026semiautonomousmathematicsdiscoverygemini,barreto2026irrationality}. Together, these results indicate that end-to-end AI-assisted mathematical discovery is becoming practically feasible.

These Erd\H{o}s-scale and benchmark-scale results show what ``specialized'' LLM systems can achieve, but they leave a practical question unresolved: can individual researchers produce mathematically substantive progress with consumer-access models and explicit human verification? This question is central for vibe-proving, i.e., sustained, iterative theorem development through conversational interaction rather than one-shot outputs. Unlike vibe-coding, where software has a runtime oracle, vibe-proving faces an intrinsic verification bottleneck because every logical step must be checked and a single hidden gap can invalidate the argument.

Our case study began when we encountered Conjecture~20 of~\cite{ran2024nonnegative} on nonreal eigenvalue regions of a structured $4\times4$ row-stochastic family of matrices. The project did not start from a pre-registered protocol, but from ``vibing'' with ChatGPT-5.1 (Thinking) on an active research problem. During that initial exploratory phase, we did not systematically retain our earliest interactions. Only after recognizing that the reasoning approach suggested by the LLM was not only interesting but also potentially viable, did we transition from informal experimentation to systematic documentation and structured iteration. Consequently, we reconstructed these initial prompts using ChatGPT-5.2 (Thinking). All subsequent documented interactions were conducted directly with ChatGPT-5.2 (Thinking). We therefore ground the paper in auditable artefacts: seven share links (\hyperref[transcript:1]{Transcript 1}--\hyperref[transcript:7]{Transcript 7}) and four proof drafts (Appendix Versions \ref{sec:version_0}--\ref{sec:version_3}), and we explicitly mark where early prompt content is reconstructed rather than preserved verbatim.

Within this artefact set, we show that a stable interaction pattern of iterating between \textit{generate, referee, and repair} resolves the conjecture into a complete characterization theorem with explicit boundary attainment. The model is strongest at proposing global structure (reparameterization and extremal strategy), whereas human effort is concentrated on correctness-critical verification (quadrant control, branch tracking, endpoint admissibility, and dense algebraic expansions). This framing complements recent mathematics results obtained with specialized or (semi-)autonomous stacks, including dedicated theorem-discovery pipelines~\cite{feng2026towards,feng2026semiautonomousmathematicsdiscoverygemini}, as well as domain-specific outputs in arithmetic geometry, combinatorics, and robust control~\cite{feng2026eigenweights,feng2026arithmetic,lee2026lower,asadi2026strongly}. Our contribution is a subscription-level, conversation-auditable account where the revision trajectory is inspectable end-to-end.

The remainder of the paper is organized as follows. Section~\ref{sec:background} introduces the mathematical setting and Conjecture~20. Section~\ref{sec:process} analyzes the interaction workflow with ChatGPT and documents the iterative proof-development process. We then discuss the division of labor between model and human verification, limitations, and implications for AI-assisted research workflows, and finally provide transcript-linked appendices and versioned proof artefacts for full auditability.

\section{Mathematical Background}
\label{sec:background}

\subsection{Spectral Regions of Stochastic Matrices}

A fundamental problem in matrix theory is to characterize which complex numbers can appear as eigenvalues of matrices with specific structural constraints. For row-stochastic matrices (nonnegative matrices with row sums equal to 1), the definitive result is the Karpelevich theorem (1951)~\cite{karpelevych1951}, which characterizes the region $K_n$ of possible eigenvalues of $n \times n$ row-stochastic matrices. For subclasses, such characterizations remain largely open.

Karpelevich's original proof is notoriously difficult~\cite{johnson2017matricial}. This difficulty has motivated significant recent work revisiting and demystifying the Karpelevich theorem~\cite{munger2024}, as well as the search for alternative approaches for subclasses of stochastic matrices where simpler characterizations might exist~\cite{guerry2022monotone,kim2022conjectures,vagenende2025eigenvalue}.

\subsection{Dimitriev and Dynkin's Trigonometric Method}

Before Karpelevich's general theorem, Dmitriev and Dynkin (1946)~\cite{dmitriev1946} developed a geometric method for characterizing the eigenvalue regions $K_n$ for small $n$. Their approach, which we call the ``trigonometric method,'' is particularly elegant:

\begin{enumerate}
    \item For an eigenvalue $\lambda$ with eigenvector $v$, derive a multiplicative constraint from the eigenvalue equations.
    \item Reparametrize using arguments ($u = \operatorname{Arg}(z + t)$) to convert the constraint into a trigonometric optimization problem.
    \item Use convexity and majorization arguments to characterize the boundary.
\end{enumerate}

Dmitriev and Dynkin successfully characterized $K_2, K_3, K_4$, and $K_5$ using this approach. Their method has the advantage of being more elementary than Karpelevich's and more amenable to adaptation for matrix subclasses with specific zero patterns.

\subsection{Conjecture 20 of Ran and Teng}

Recently, Ran and Teng (2024)~\cite{ran2024nonnegative} studied eigenvalue regions for matrices with prescribed zero patterns, obtaining complete characterizations for the $3 \times 3$ case. They also posed two conjectures for dimension four, of which Conjecture~20 concerns what we will call the ``4-cycle'' pattern.

Consider the family of $4 \times 4$ row-stochastic matrices:
\begin{equation}
\label{eq:matrix-family}
A(\alpha, \beta, \gamma, \delta) = \begin{pmatrix}
\alpha & 1-\alpha & 0 & 0 \\
0 & \beta & 1-\beta & 0 \\
0 & 0 & \gamma & 1-\gamma \\
1-\delta & 0 & 0 & \delta
\end{pmatrix}, \quad \alpha, \beta, \gamma, \delta \in (0,1).
\end{equation}

This matrix has a so called cyclic zero pattern: state 1 transitions to states 1 or 2; state 2 to 2 or 3; state 3 to 3 or 4; state 4 to 4 or 1. Such patterns arise naturally in the study of Markov chains with restricted transitions.

Ran and Teng (2024)~\cite{ran2024nonnegative} defined the region $R$ as:
\begin{equation}
R = \{z = a + bi \in K_4 : a > 0, G(a, b) > 0\}
\end{equation}
where $K_4$ is the Karpelevich region for $4 \times 4$ matrices and
\begin{equation}
G(a, b) = (b^2 + a^2 + a)^2 + 2a^2 - b^2.
\end{equation}

\setcounter{theorem}{19}
\begin{conjecture}[Ran and Teng, 2024~\cite{ran2024nonnegative}, Conjecture~20]
\label{conj:ran-teng}
For any irreducible matrix $A$ of the form (1), the non-real eigenvalues of $A$ lie in the region $R$.
\end{conjecture}

Ran and Teng (2024)~\cite{ran2024nonnegative} provided substantial numerical evidence for this conjecture and proved that the curve $G(a, b) = 0$ is indeed attained by the one-parameter subfamily $A_L(\alpha)$ where $\beta = \gamma = \delta = 0$:
\begin{equation}
A_L(\alpha) = \begin{pmatrix}
\alpha & 1-\alpha & 0 & 0 \\
0 & 0 & 1 & 0 \\
0 & 0 & 0 & 1 \\
1 & 0 & 0 & 0
\end{pmatrix}.
\end{equation}

They conjectured that this curve forms the left boundary of the non-real part of the spectral region.

\subsection{Contribution}

We observed that Ran and Teng's (2024)~\cite{ran2024nonnegative} Conjecture~20 could be used to provide a complete characterization. Using the Karpelevich theorem and results on realizing boundary arcs (specifically, that the right boundary of $K_4$ is already attained within our restricted class), we provided a target theorem statement and asked ChatGPT for a proof strategy.

The collaboration produced a proof of a theorem characterizing all non-real eigenvalues of the matrix family (1): necessary and sufficient conditions, plus explicit boundary attainment. The full proof, developed through extended dialogue with ChatGPT, will soon be available on the ArXiv. For this paper, the goal is to describe the process, not the final product.

\section{The Conversations with ChatGPT}
\label{sec:process}

\subsection{Casual testing of GPT's capabilities}

The interaction with ChatGPT did not begin with a research plan. We were ``vibing'' by casually testing the ``thinking extended'' mode on problems from our research area. During these informal explorations, we observed that the reasoning approach suggested by the model was not only interesting but also potentially viable. Because these initial tests were exploratory, we did not systematically document our earliest interactions with ChatGPT-5.1 (Thinking). To compensate for the missing conversations with ChatGPT-5.1 (Thinking), we reran the initial prompts with ChatGPT-5.2 (Thinking) in \hyperref[transcript:1]{Transcript 1} and \hyperref[transcript:2]{Transcript 2}. From this point onward, all formal testing and documented interactions were performed exclusively using ChatGPT-5.2 (Thinking), starting from \hyperref[transcript:3]{Transcript 3}. What follows is a reconstruction with the versioned drafts (Appendix Versions \ref{sec:version_0}--\ref{sec:version_3}).

\subsection{Setting Up the Theorem}

We provided the Ran and Teng paper~\cite{ran2024nonnegative} and explained our observation about the right boundary being already characterized by Karpelevich~\cite{karpelevych1951}. This resulted in the characterization Theorem below. Afterwards we provided ChatGPT with this theorem statement and prompted the model towards the direction of thinking in line with the trigonometric method of Dimitriev and Dynkin~\cite{dmitriev1946characteristic}, and asked for a proof strategy (see \hyperref[transcript:1]{Transcript 1} for the exact statement we used in the prompt):

\begin{theorem*}[Spectral region for a 4-cycle row-stochastic matrix]
Let
\[
A(\alpha,\beta,\gamma,\delta)=
\begin{pmatrix}
\alpha & 1-\alpha & 0 & 0\\
0 & \beta & 1-\beta & 0\\
0 & 0 & \gamma & 1-\gamma\\
1-\delta & 0 & 0 & \delta
\end{pmatrix},
\qquad \alpha,\beta,\gamma,\delta\in[0,1].
\]
Let $\lambda=a+ib\in\sigma(A(\alpha,\beta,\gamma,\delta))$ with $b\neq 0$, and write $b_+=|b|$.
Then:
\begin{enumerate}
\item $0\le a\le 1$;
\item $a+b_+\le 1$;
\item with
\[
G(a,b_+):=(b_+^{2}+a^{2}+a)^{2}+2a^{2}-b_+^{2},
\]
we have $G(a,b_+)\ge 0$.
\end{enumerate}

Conversely, if $b>0$ and
\[
0\le a\le 1,\qquad a+b<1,\qquad G(a,b)\ge 0,
\]
then there exist $\alpha,\beta,\gamma,\delta\in[0,1]$ such that $\lambda\in\sigma(A(\alpha,\beta,\gamma,\delta))$.
(The lower half-plane follows by conjugation.)

Moreover:
\begin{itemize}
\item the segment $C_R:\ \lambda=1-x+ix$ ($x\in[0,1]$) is attained by $\alpha=\beta=\gamma=\delta=1-x$;
\item the curve portion $C_L$ of $G(a,b)=0$ in the upper half-plane joining $i$ to $0$ is attained by
\[
A_L(\alpha)=
\begin{pmatrix}
\alpha & 1-\alpha & 0 & 0\\
0&0&1&0\\
0&0&0&1\\
1&0&0&0
\end{pmatrix},\qquad \alpha\in[0,1];
\]
\item every point strictly inside $\{b>0,\ 0\le a\le 1,\ a+b<1,\ G(a,b)\ge 0\}$ is attained.
\end{itemize}
\end{theorem*}

\subsection{The Proof Strategy}

ChatGPT then proposed a trigonometric roadmap (angle parametrization plus Jensen/Karamata) inspired by~\cite{dmitriev1946characteristic,swift1972location} for the remaining constraints (see \hyperref[transcript:1]{Transcript 1}). The next prompt drew on our knowledge of the classical literature (see \hyperref[transcript:1]{Transcript 1}). The response outlined the trigonometric approach: the multiplicative constraint from the eigenvalue equations, reparametrization by arguments $u = \operatorname{Arg}(z+t)$, the convex function $F$, and the use of Jensen's inequality and Karamata's inequality for the optimization. We then requested a hypercritical review, which surfaced a fatal gap in the constrained Karamata/majorization step (see \hyperref[transcript:2]{Transcript 2}).

\subsection{The Iterations}

Appendix Version~\ref{sec:version_0} contained the right structure but had gaps (see \hyperref[transcript:3]{Transcript 3}): incorrect quadrant handling for $\arctan$, a flawed inequality in what became Lemma~2, and incomplete algebra connecting the optimization to the curve $G(a,b) = 0$.

The first correctness-critical bug was a feasibility claim in what became Lemma~2: the model's original construction fails in the near-endpoint regime $M\downarrow \pi/2$. The repair was to localize the argument near the endpoint configuration $u_4\uparrow M$, and to require only the strict interior condition needed for that regime (namely $u_1<M$) rather than uniform slack up to $M$. This ``local feasibility by continuity'' patch is documented in Appendix Version~\ref{sec:version_1} and corroborated by an independent referee pass (see \hyperref[transcript:6]{Transcript 6}).

ChatGPT helped simplify and confirm the factorization $|\lambda|^6 \geq N(a,b)$ (see \hyperref[transcript:5]{Transcript 5}). After some manipulation, it produced:
\begin{equation}
|\lambda|^6 - N(a,b) = |\lambda - 1|^2 G(a,b).
\end{equation}

A disproportionate share of the iteration time went into Lemma~4 (see \hyperref[transcript:5]{Transcript 5}), which forces the ``tight regime'' $3m+M>2\pi$ from the algebraic condition $G(a,b)\le 0$. To close this step, we wrote a fully expanded Lamport-style draft of Lemma~4 and ran a dedicated ``referee'' thread focused on finding silent branch/sign errors rather than proposing new structure. This pass largely validated the backbone of the argument (quadratic-in-$s=b^2$ analysis, the deduction $b^2>3a^2$ locating $m$ in the correct range, reduction of $3m+M>2\pi$ to a tangent inequality, and the load-bearing algebraic identity $|\lambda|^6-N(a,b)=|\lambda-1|^2G(a,b)$), but it also surfaced several obligations that earlier drafts had left implicit: (i) the discriminant condition for $G$ as a quadratic in $s$ is $\Delta\ge 0 \iff -\tfrac12\le a\le \tfrac16$, so the conclusion $a\le\tfrac16$ must be explicitly combined with the already-established constraint $a\ge 0$; (ii) the $\Arg/\arctan/\tan$ branch conventions and the monotonicity interval for $\tan$ must be stated once to justify the conversion of $3m+M>2\pi$ into a tangent inequality; and (iii) every squaring/cross-multiplication step must be preceded by a one-line positivity check (notably $1-a>0$ and denominator signs). We incorporated these items as explicit guard lines and reorganized Lemma~4 into a ``sign-check $\rightarrow$ transform $\rightarrow$ square/cross-multiply'' micro-step schedule, which made the lemma locally checkable and prevented regressions in later rewrites.

Appendix Version~\ref{sec:version_2} is the first draft where the correctness-critical obligations are made explicit and discharged end-to-end (see \hyperref[transcript:4]{Transcript 4}). Appendix Version~\ref{sec:version_3} then primarily improves auditability: it expands the load-bearing algebraic derivations (notably the factorization $|\lambda|^{6}-N(a,b)=|\lambda-1|^{2}G(a,b)$; see \hyperref[transcript:5]{Transcript 5}) and rewrites the argument in a Lamport-style dependency structure (see \hyperref[transcript:7]{Transcript 7}).

\subsection{From chronology to workflow-level lessons}
\label{sec:process_bridge}
The remainder of the paper abstracts the above chronology into workflow-level observations about verification, patch search, and division of labor. To keep Section~\ref{sec:process} chronological, we consolidate those methodological takeaways in the Discussion (Section~\ref{sec:reflections}), with explicit pointers to the corresponding transcripts and appendix versions.

\section{Discussion}
\label{sec:reflections}

\subsection{What This Case Suggests}

In this case study, a conversational workflow with ChatGPT-5.2 (Thinking) produced a complete and checkable proof of Ran and Teng (2024)~\cite{ran2024nonnegative} Conjecture~20. The LLM supplied a viable global architecture early (characterization theorem + trigonometric strategy), while the main human work consisted of identifying and discharging correctness obligations (quadrant/branch handling, endpoint cases, and long algebra).

Success in this case used the following features of the problem and workflow:

\begin{enumerate}
\item \textbf{Prior scaffolding:} The target was an explicit region characterization with known boundary information (Karpelevich right boundary and the conjectured left boundary $G(a,b)=0$), and a classical proof template existed via a Dmitriev--Dynkin-style trigonometric reduction.
\item \textbf{Two anchoring prompts:} The interaction was organized around (i) providing the target characterization theorem statement and boundary context and (ii) deriving a trigonometric/majorization-based proof strategy specialized to the 4-cycle matrix family.
\item \textbf{Untrusted outputs:} Early drafts contained correctness-critical errors (e.g., $\arctan$ quadrant handling and a false inequality regime in Lemma~2), so every model-generated step was treated as a candidate and checked before inclusion.
\item \textbf{Independent patch search:} When a gap was identified, we queried multiple independent sessions and compared proposed fixes, adopting only patches that could be verified and rejecting incompatible derivations.
\end{enumerate}

We unpack the concrete verification workflow and division of labor in Sections \ref{sec:workflow}--\ref{sec:division}.

\subsection{Workflow and Verification Strategy}
\label{sec:workflow}
Across the seven threads, we converged on a stable loop of generate, referee, repair: generate candidate steps and proof skeletons; run referee-style passes to surface hidden obligations (branch conventions, positivity checks before squaring/cross-multiplying, endpoint admissibility); and then repair targeted subclaims without rewriting unrelated parts.

Three concrete practices materially improved reliability:
\begin{itemize}
    \item \textbf{Parallel patch search.} When a gap was identified, we opened multiple independent sessions with the same obligation and treated their outputs as competing patches. We adopted only patches that were independently checkable (or where independent sessions converged), and rejected fixes that relied on silent case splits or unverifiable algebraic compression.
    \item \textbf{Fresh-session referee passes (bounded).} At checkpoints, we pasted the then-current draft into a fresh thread and asked for a gap list (missing assumptions, unjustified monotonicity/branch steps, unsafe squaring). Once major gaps were fixed, repeated ``hypercritical'' prompting showed diminishing returns (stylistic noise and occasional false positives), so we used it as a bounded diagnostic rather than an endless loop.
    \item \textbf{Regression control via versioning and dependency visibility.} Proof rewrites sometimes reintroduced errors elsewhere (notation drift, dropped sign constraints). We therefore kept versioned drafts aligned with transcript excerpts and re-checked downstream dependencies after each patch. We also frequently requested Lamport-style claim decomposition~\cite{lamport1993} to make dependencies explicit and to localize verification (see \hyperref[transcript:7]{Transcript 7}).
\end{itemize}

\subsection{Division of labor and the verification bottleneck}
\label{sec:division}
In this project, the model's highest leverage was proposing global structure (reparameterizations, extremal strategies, candidate factorizations) and producing fast symbolic manipulations; human time was dominated by correctness-critical closure (branch/quadrant tracking, sign checks before squaring, endpoint admissibility, and long expansions).

Two observations were consistent across iterations:
\begin{enumerate}
    \item \textbf{The bottleneck is narrow but expensive.} Most iteration time concentrated on a small number of obligations (notably the tight-regime step around Lemma~4 and a few load-bearing factorizations), where a single silent sign/branch error can invalidate the argument.
    \item \textbf{Mechanization would directly target the slowest work.} The dominant verification load was algebraic expansion/simplification and inequality-domain checking. These are well matched to CAS and certified inequality/interval checkers, suggesting a practical hybrid pipeline: LLMs propose patches for explicit obligations; mechanized tooling validates the algebra/inequality substeps.
\end{enumerate}

For completeness, we summarize the division of contributions (with transcript anchors):
\medskip

\noindent\textbf{Human contributions (verification and orchestration).}
\begin{itemize}
    \item Problem selection and target theorem specification (see \hyperref[transcript:1]{Transcript 1}).
    \item Error detection and obligation listing (quadrant/branch handling, Lemma~2 feasibility, endpoint issues).
    \item Orchestration of independent patch search and regression control via versioned drafts.
    \item Final correctness closure on sign/branch conditions and expanded algebra (see \hyperref[transcript:5]{Transcript 5} and \hyperref[transcript:6]{Transcript 6}).
\end{itemize}

\noindent\textbf{ChatGPT contributions (structure and candidate derivations).}
\begin{itemize}
    \item Early global roadmap (trigonometric reduction + extremal strategy; \hyperref[transcript:1]{Transcript 1}).
    \item Algebraic manipulation support, including confirming the key factorization (\hyperref[transcript:5]{Transcript 5}).
    \item Structured Lamport-style rewrite that made dependencies explicit (\hyperref[transcript:7]{Transcript 7}).
    \item Targeted referee passes that surfaced missing guard conditions when prompted appropriately (\hyperref[transcript:2]{Transcript 2}).
\end{itemize}

\subsection{Formal Proof Assistants: The Verification Alternative}

A natural question is how vibe proving with LLMs compares to formal proof assistants such as Lean, Coq, or Isabelle~\cite{coqdevteam2024coq,demoura2021lean4}. These systems offer something LLMs cannot: mechanically verified correctness relative to a small logical kernel. When a Lean proof checks, the result is certain (modulo trust in the kernel, which is small but not zero, as bugs have been found in proof assistant kernels, and complete formal verification of the kernel itself remains an open problem)~\cite{ospanov2025apollo}.

We did not attempt a formalization of the present proof. A mechanically verified proof would require encoding the trigonometric optimization, convexity/majorization steps, and the inequality-heavy algebra within a proof assistant's libraries. This would yield stronger guarantees, but it would also change the artefact (proof script rather than narrative proof) and require additional engineering beyond the scope of this case study.

Our output is a conventional, human-checked proof accompanied by an auditable chat trail.

\subsection{Limitations}

\begin{enumerate}
\item \textbf{Novelty:} The proof strategy stayed within a classical Dmitriev--Dynkin/Jensen--Karamata template. The model's main contribution was assembling and adapting this template to the 4-cycle matrix family rather than introducing a fundamentally new method.
\item \textbf{Error modes:} We observed recurrent correctness failures in early drafts, including inverse-trigonometric branch/quadrant mistakes, missing sign conditions before squaring, and ``compressed'' algebra that omitted required intermediate steps.
\item \textbf{Verification bottleneck:} The slowest component was discharging a small set of technical obligations (notably Lemma~4's inequality/expansion chain and endpoint admissibility), rather than generating candidate derivations.
\item \textbf{Scope:} Evidence here is limited to one structured spectral-region characterization problem with strong prior scaffolding (known conjectured boundary and known trigonometric reduction). We do not test problems lacking such structure.
\end{enumerate}

\subsection{Recommendations for Practitioners}

The workflow-level observations in Sections \ref{sec:workflow}--\ref{sec:division} suggest the following checklist for vibe proving:
\begin{enumerate}
\item \textbf{Start from scaffolding.} Prefer problems where you can state a concrete target theorem and where a recognizable reduction/template exists.
\item \textbf{Turn critique into obligations.} Convert ``this seems wrong'' into an explicit obligation list (domains, branch conventions, positivity checks before squaring, endpoints).
\item \textbf{Use parallel patch search.} Treat independent sessions as competing patch generators; adopt only patches that you can verify locally.
\item \textbf{Control regressions.} Keep versioned drafts and re-check downstream dependencies after each patch; prefer Lamport-style decomposition to expose dependencies.
\item \textbf{Mechanize the bottleneck.} Offload expansions and inequality-domain checks to CAS / certified checkers; reserve human time for conceptual choices and boundary cases.
\end{enumerate}

\clearpage

\section*{Acknowledgements}
This research was supported by funding from the Flemish Government under the ``Onderzoeksprogramma Artifici\"ele Intelligentie (AI) Vlaanderen'' program.
Andres Algaba acknowledges support from the Francqui Foundation (Belgium) through a Francqui Start-Up Grant and a fellowship from the Research Foundation Flanders (FWO) under Grant No.1286924N. 
Vincent Ginis acknowledges support from Research Foundation Flanders under Grant No.G032822N and G0K9322N. 

\section*{Author contributions}
All authors collaboratively conceived the main idea of the study. Brecht Verbeken co-constructed the proof with the assistance of ChatGPT-5.1 (Thinking) and ChatGPT-5.2 (Thinking). Brecht Verbeken, Brando Vagenende, and Marie-Anne Guerry reviewed and validated the final proof. Andres Algaba and Brecht Verbeken drafted the manuscript. All authors collaboratively revised the manuscript and provided critical feedback.

\section*{Data and code availability}
\subsection*{Conversations}

\begin{itemize}[leftmargin=*,itemsep=2pt,topsep=2pt]
\item \hyperref[transcript:1]{Transcript 1} (rerun): setup prompt + roadmap; includes follow-up algebra sanity checks (via ChatGPT's Python tool).
\item \hyperref[transcript:2]{Transcript 2} (rerun): hypercritical review; identifies the majorization/Karamata gap.
\item \hyperref[transcript:3]{Transcript 3}:referee pass on the first full proof attempt (Appendix Version \ref{sec:version_0}) and early gap list (branch/quadrant handling; Lemma~2 feasibility near \(M\); missing algebra).
\item \hyperref[transcript:4]{Transcript 4} corresponds to the first essentially complete draft (Appendix Version \ref{sec:version_2}).
\item \hyperref[transcript:5]{Transcript 5}: referee pass on Lemma~4 (tight-regime step); checks branch/sign conditions and validates the factorization $|\lambda|^{6}-N(a,b)=|\lambda-1|^{2}G(a,b)$.
\item \hyperref[transcript:6]{Transcript 6}: independent fresh-session critique corroborating the branch/quadrant and Lemma~2 issues/fix.
\item \hyperref[transcript:7]{Transcript 7}: Lamport-style rewrite used for the final exposition (Appendix Version \ref{sec:version_3}).
\end{itemize}

Note that the following link template can be used to inspect meta-data from each conversation: \url{https://chatgpt.com/backend-api/share/<share-id>}.

\begin{itemize}[leftmargin=7.5pt]
\item \textbf{Transcript 1 (rerun):}\label{transcript:1} \url{https://chatgpt.com/share/699464b2-c81c-8002-9ba2-bad952e6414a}
\item \textbf{Transcript 2 (rerun):}\label{transcript:2} \url{https://chatgpt.com/share/69946540-e684-8002-a8ca-feb45e4da7be}
\item \textbf{Transcript 3:}\label{transcript:3} 
\url{https://chatgpt.com/share/697bb2d1-0418-8007-8f2a-474e3e6430fa}
\item \textbf{Transcript 4:}\label{transcript:4} \url{https://chatgpt.com/share/697bb2bd-5450-8007-9115-92589d95e1ba}
\item \textbf{Transcript 5:}\label{transcript:5} \url{https://chatgpt.com/share/697bb29a-c190-8007-b37e-87b390e9d9ff}
\item \textbf{Transcript 6:}\label{transcript:6} \url{https://chatgpt.com/share/697bb288-010c-8007-972b-ec676e9a16f7}
\item \textbf{Transcript 7:}\label{transcript:7} \url{https://chatgpt.com/share/697bb2ab-f83c-8007-a8fc-9439dd3d7488}
\end{itemize}

\clearpage
\bibliographystyle{unsrt} 
\bibliography{references}

\clearpage
Here, we show the subsequent versions of the proof.

\appendix
\makeatletter
\@addtoreset{equation}{section}
\makeatother
\renewcommand{\theequation}{\arabic{equation}}
\renewcommand{\theHequation}{app.\thesection.\arabic{equation}}
\section{Appendix}
\label{sec:version_0}

\begin{versionzerobox}
\begin{theorem}[Precise spectral region for a 4-cycle row-stochastic matrix]
Let
\[
A(\alpha,\beta,\gamma,\delta)=
\begin{pmatrix}
\alpha & 1-\alpha & 0 & 0\\
0 & \beta & 1-\beta & 0\\
0 & 0 & \gamma & 1-\gamma\\
1-\delta & 0 & 0 & \delta
\end{pmatrix},
\qquad \alpha,\beta,\gamma,\delta\in[0,1).
\]
Let $\lambda=a+ib\in\sigma(A(\alpha,\beta,\gamma,\delta))$ with $b\neq 0$, and write $b_+=|b|$.
Then:

\begin{enumerate}
\item $0\le a\le 1$.
\item $a+b_+\le 1$.
\item with
\[
G(a,b_+):=(b_+^{2}+a^{2}+a)^{2}+2a^{2}-b_+^{2},
\]
we have $G(a,b_+)\ge 0$.
\end{enumerate}

Conversely, if $b>0$ and
\[
0\le a\le 1,\qquad a+b<1,\qquad G(a,b)>0,
\]
then there exist $\alpha,\beta,\gamma,\delta\in[0,1)$ such that
$\lambda\in\sigma(A(\alpha,\beta,\gamma,\delta))$.
(The lower half-plane follows by conjugation.)

Moreover:
\begin{itemize}
\item the segment $C_R:\ \lambda=1-x+ix$ ($x\in[0,1]$) is attained by $\alpha=\beta=\gamma=\delta=1-x$;
\item the curve portion $C_L$ of $G(a,b)=0$ in the upper half-plane joining $i$ to $0$ is attained by
\[
A_L(\alpha)=
\begin{pmatrix}
\alpha & 1-\alpha & 0 & 0\\
0&0&1&0\\
0&0&0&1\\
1&0&0&0
\end{pmatrix},\qquad \alpha\in[0,1);
\]
\item every point strictly inside $\{b>0,\ 0\le a\le 1,\ a+b<1,\ G(a,b)>0\}$ is attained.
\end{itemize}
\end{theorem}

\begin{proof}
\textbf{1) Eigenvalue equation.}
Let $v\neq 0$ satisfy $Av=\lambda v$. Row-by-row,
\[
(\lambda-\alpha)v_1=(1-\alpha)v_2,\quad
(\lambda-\beta)v_2=(1-\beta)v_3,\quad
(\lambda-\gamma)v_3=(1-\gamma)v_4,\quad
(\lambda-\delta)v_4=(1-\delta)v_1.
\]
Since $1-\alpha,\dots,1-\delta>0$, if any $v_k=0$ then all $v_j=0$, contradiction. Hence
$v_1v_2v_3v_4\neq 0$, and multiplying gives
\begin{equation}
(\lambda-\alpha)(\lambda-\beta)(\lambda-\gamma)(\lambda-\delta)
=(1-\alpha)(1-\beta)(1-\gamma)(1-\delta).
\label{eq:1_v1}
\end{equation}
Put $t_1=1-\alpha,\dots,t_4=1-\delta\in(0,1]$ and $z=\lambda-1$. Then
$\lambda-(1-t_k)=z+t_k$, so Equation \ref{eq:1_v1} becomes
\begin{equation}
(z+t_1)(z+t_2)(z+t_3)(z+t_4)=t_1t_2t_3t_4.
\label{eq:2_v1}
\end{equation}

\medskip
\textbf{2) Angle parametrization and the function $F$.}
Assume $b>0$ (the case $b<0$ follows by conjugation since $A$ is real). Write
\[
z=x+iy,\qquad x=a-1,\ y=b.
\]
For $t>0$, $z+t$ lies in the upper half-plane, so define $u(t)=\argu(z+t)\in(0,\pi)$.
Elementary trigonometry gives the inverse relation
\begin{equation}
t=t(u)=y\cot u-x,\qquad |z+t(u)|=y\csc u.
\label{eq:3_v1}
\end{equation}
Define
\begin{equation}
F(u):=\log|z+t(u)|-\log t(u)=\log(y\csc u)-\log(y\cot u-x).
\label{eq:4_v1}
\end{equation}
Let
\begin{equation}
m:=\argu(z+1)=\argu(\lambda),\qquad M:=\argu(z)=\argu(\lambda-1).
\label{eq:5_v1}
\end{equation}
As $t\downarrow 0$, $u(t)\uparrow M$; as $t\uparrow 1$, $u(t)\downarrow m$. Thus
\begin{equation}
t\in(0,1]\quad\Longleftrightarrow\quad u\in[m,M).
\label{eq:6_v1}
\end{equation}

\begin{lemma}[Log-modulus/argument reformulation]
A nonreal $\lambda$ satisfies Equation \ref{eq:2_v1} for some $t_k\in(0,1]$ iff there exist $u_k\in[m,M)$ such that
\begin{equation}
u_1+u_2+u_3+u_4=2\pi,
\label{eq:7_v1}
\end{equation}
\begin{equation}
F(u_1)+F(u_2)+F(u_3)+F(u_4)=0.
\label{eq:8_v1}
\end{equation}
\end{lemma}

\begin{proof}
Write $z+t_k=|z+t_k|e^{iu_k}$ with $u_k\in(0,\pi)$. Taking arguments in Equation \ref{eq:2_v1} gives
$\sum u_k\equiv 0\pmod{2\pi}$; since $\sum u_k\in(0,4\pi)$ this forces $\sum u_k=2\pi$, i.e. Equation \ref{eq:7_v1}.
Taking moduli and logs gives Equation \ref{eq:8_v1}. Condition Equation \ref{eq:6_v1} is exactly $t_k\in(0,1]$.
Conversely, Equation \ref{eq:7_v1} and Equation \ref{eq:8_v1} give equality of arguments and moduli, hence Equation \ref{eq:2_v1}.
\end{proof}

Define
\[
\mathcal P:=\Big\{(u_1,\dots,u_4)\in[m,M)^4:\ u_1+\cdots+u_4=2\pi\Big\},
\qquad
\Psi(u_1,\dots,u_4):=\sum_{k=1}^4F(u_k).
\]
Then $\lambda$ is a nonreal eigenvalue of some $A(\alpha,\beta,\gamma,\delta)$ iff
$\mathcal P\neq\emptyset$ and $0\in\Psi(\mathcal P)$.

\medskip
\textbf{3) Necessarily $0\le a\le 1$ (indeed $a<1$ if $b\ne 0$).}
Every $(u_1,\dots,u_4)\in\mathcal P$ has average $\pi/2$, so $\pi/2\in[m,M)$, i.e.
\begin{equation}
m\le \frac{\pi}{2} < M.
\label{eq:9_v1}
\end{equation}
For $b>0$, $m\le\pi/2\iff a\ge 0$. Also $M>\pi/2\iff a<1$ because $\Re(\lambda-1)=a-1<0$
iff $a<1$. Thus $\mathcal P\neq\emptyset$ implies
\begin{equation}
0\le a<1,
\label{eq:10_v1}
\end{equation}
hence in particular $0\le a\le 1$.

\medskip
\textbf{4) Strict convexity of $F$.}
Differentiate Equation \ref{eq:4_v1} using Equation \ref{eq:3_v1}. A direct computation yields
\begin{equation}
F''(u)=\frac{x^2+y^2}{(x-y\cot u)^2\sin^2 u}>0\qquad(u\in(0,\pi)),
\label{eq:11_v1}
\end{equation}
so $F$ is strictly convex on $(0,\pi)$, hence on $[m,M)$.

\medskip
\textbf{5) Right boundary: $a+b\le 1$, and attainment of $C_R$.}
$\mathcal P$ is convex and connected; $\Psi$ is continuous, hence $\Psi(\mathcal P)$ is an interval.
By Jensen and strict convexity,
\begin{equation}
\Psi(u_1,\dots,u_4)\ge 4F\!\Big(\frac{u_1+\cdots+u_4}{4}\Big)=4F(\pi/2).
\label{eq:12_v1}
\end{equation}
Thus $0\in\Psi(\mathcal P)$ forces $F(\pi/2)\le 0$. Now
\[
F(\pi/2)=\log(y)-\log(-x)=\log\!\Big(\frac{b}{1-a}\Big),
\]
so $F(\pi/2)\le0\iff b\le 1-a\iff a+b\le 1$. This proves item (2) (for $b>0$; for $b<0$ replace $b$ by $|b|$).

Attainment of $C_R$: if $t_1=t_2=t_3=t_4=x\in(0,1]$, then Equation \ref{eq:2_v1} becomes $(z+x)^4=x^4$, i.e.
$\lambda=1-x\pm ix$, attained by $\alpha=\beta=\gamma=\delta=1-x$.

\medskip
\textbf{6) A max principle on $\mathcal P$.}
There are two regimes.

\begin{lemma}[Unbounded supremum]
If
\begin{equation}
3m+M\le 2\pi,
\label{eq:13_v1}
\end{equation}
then $\sup_{\mathcal P}\Psi=+\infty$.
\end{lemma}

\begin{proof}
Take $u_4\uparrow M$ and $u_1=u_2=u_3=(2\pi-u_4)/3$. Condition Equation \ref{eq:13_v1} ensures $u_1\ge m$.
Also $u_4<M$ and $u_1\le (2\pi-M)/3\le M$ since $M\ge\pi/2\Rightarrow 2\pi\le 4M$. Hence
$(u_1,\dots,u_4)\in\mathcal P$.
As $u_4\uparrow M$, $t(u_4)=y\cot u_4-x\downarrow 0$, so $-\log t(u_4)\to+\infty$ while
$\log(y\csc u_4)$ stays bounded; thus $F(u_4)\to+\infty$ and $\Psi\to+\infty$.
\end{proof}

\begin{lemma}[Finite maximum in the tight regime]
If
\begin{equation}
3m+M>2\pi,
\qquad U:=2\pi-3m<M,
\label{eq:14_v1}
\end{equation}
then
\begin{equation}
\max_{\mathcal P}\Psi = 3F(m)+F(U),
\label{eq:15_v1}
\end{equation}
and equality holds iff $(u_1,\dots,u_4)$ is a permutation of $(U,m,m,m)$.
\end{lemma}

\begin{proof}
Order $u_k$ decreasing: $v_1\ge v_2\ge v_3\ge v_4$. Since each $v_j\ge m$ and $\sum v_j=2\pi$,
\[
v_1=2\pi-(v_2+v_3+v_4)\le 2\pi-3m=U.
\]
Thus $(v_1,\dots,v_4)\in[m,U]^4$ with sum $2\pi$. The vector $(U,m,m,m)$ majorizes $(v_1,\dots,v_4)$.
By Karamata (convex $F$),
\[
\sum_{j=1}^4F(v_j)\le F(U)+3F(m),
\]
with equality only at permutations of $(U,m,m,m)$ since $F$ is strictly convex.
\end{proof}

\medskip
\textbf{7) The inequality $G(a,b)\ge 0$.}
We prove necessity for every nonreal eigenvalue.

\subsubsection*{7.1 Geometry lemma: $G(a,b)\le 0\Rightarrow 3m+M>2\pi$}

\begin{lemma}
Assume $b>0$ and $0\le a\le 1$. If $G(a,b)\le 0$, then $3m+M>2\pi$.
\end{lemma}

\begin{proof}
Set $s=b^2$. Then
\begin{equation}
G(a,b)=s^2+s(2a^2+2a-1)+(a^2+a)^2+2a^2.
\label{eq:16_v1}
\end{equation}
This is quadratic in $s$ with discriminant
\[
\Delta=(2a^2+2a-1)^2-4\big((a^2+a)^2+2a^2\big)=-(2a+1)(6a-1).
\]
Thus $G(a,b)\le 0$ forces $\Delta>0$, hence $a<1/6$, and $s$ lies between the two real roots; in particular
\begin{equation}
s\ge s_-(a):=\frac12-a-a^2-\frac12\sqrt{(1-6a)(1+2a)}.
\label{eq:17_v1}
\end{equation}

\emph{Case 1: $a=0$.} Then $m=\argu(ib)=\pi/2$, and $M=\argu(-1+ib)=\pi-\arctan(b)$. Hence
\[
3m+M=\frac{3\pi}{2}+\pi-\arctan(b)=\frac{5\pi}{2}-\arctan(b)>2\pi.
\]

\emph{Case 2: $a>0$.} From Equation \ref{eq:17_v1} we get $s>s_-(a)\ge 3a^2$, with strictness because
\[
s_-(a)-3a^2=\frac12-a-4a^2-\frac12\sqrt{(1-6a)(1+2a)}
\]
satisfies $(\text{LHS})^2-(\text{RHS})^2=32a^3(2a+1)>0$ for $a>0$. Hence $b^2>3a^2$, i.e.
$\tan m=b/a>\sqrt3$, so $m>\pi/3$ and $3m-\pi\in(0,\pi/2)$.

Write $\phi:=\arctan\!\big(\frac{b}{1-a}\big)\in(0,\pi/2)$, so $M=\pi-\phi$. Then
\begin{equation}
3m+M>2\pi
\ \Longleftrightarrow\
3m-\pi>\phi
\ \Longleftrightarrow\
\tan(3m)>\frac{b}{1-a},
\label{eq:18_v1}
\end{equation}
since $\tan$ is increasing on $(0,\pi/2)$ and $\tan(3m-\pi)=\tan(3m)$.

Let $t=\tan m=b/a$ (allowed since $a>0$). The triple-angle identity gives
$\tan(3m)=\frac{3t-t^3}{1-3t^2}$, and a direct simplification yields
\begin{equation}
\tan(3m)-\frac{b}{1-a}
=\frac{b\,(4a^3-3a^2-4ab^2+b^2)}{a(a-1)(a^2-3b^2)}.
\label{eq:19_v1}
\end{equation}
Here $a>0$, $a-1<0$, and $a^2-3b^2<0$ (since $b^2>3a^2$), so the denominator in Equation \ref{eq:19_v1} is positive.
Hence the sign equals the sign of
\begin{equation}
N(a,b):=4a^3-3a^2-4ab^2+b^2 = (1-4a)b^2 + a^2(4a-3).
\label{eq:20_v1}
\end{equation}
For $a<1/6$, $1-4a>0$ and $4a-3<0$, so $N(a,b)$ is strictly increasing in $b^2$ and has a unique zero at
\begin{equation}
s_0(a)=\frac{a^2(3-4a)}{1-4a}.
\label{eq:21_v1}
\end{equation}
One checks that $s_-(a)>s_0(a)$ for every $a\in(0,1/6)$: after multiplying by $1-4a>0$ and squaring,
the difference becomes $256a^6>0$. Therefore $b^2\ge s_-(a)>s_0(a)$ implies $N(a,b)>0$.
By Equation \ref{eq:19_v1}--\ref{eq:18_v1}, this gives $3m+M>2\pi$.
\end{proof}

\subsubsection*{7.2 Tight regime: $\max\Psi\ge 0\iff G\ge 0$}

Assume $3m+M>2\pi$, so Lemma~6.2 applies with $U=2\pi-3m\in(0,\pi)$ and maximizer $(U,m,m,m)$.
Since $u(1)=\argu(z+1)=m$, Equation \ref{eq:6_v1} gives $t(m)=1$, hence
\begin{equation}
F(m)=\log|z+1|-\log 1=\log|\lambda|.
\label{eq:22_v1}
\end{equation}
Also $t(U)=b\cot U-(a-1)=b\cot U+1-a$ and $|z+t(U)|=b\csc U$, so by Equation \ref{eq:15_v1}
\begin{equation}
\max_{\mathcal P}\Psi
=3\log|\lambda|+\log(b\csc U)-\log(b\cot U+1-a)
=\log\!\left(\frac{|\lambda|^3\,b\csc U}{b\cot U+1-a}\right).
\label{eq:23_v1}
\end{equation}
Thus $\max_{\mathcal P}\Psi\ge 0$ is equivalent to
\begin{equation}
|\lambda|^3\,b\csc U \ \ge\ b\cot U+1-a.
\label{eq:24_v1}
\end{equation}

Now $U=2\pi-3m$, so $\sin U=-\sin(3m)$ and $\cos U=\cos(3m)$. Since
$\lambda^3=(a+ib)^3=(a^3-3ab^2)+i(3a^2b-b^3)$, we have
\begin{equation}
\cos(3m)=\frac{\Re(\lambda^3)}{|\lambda|^3}=\frac{a(a^2-3b^2)}{|\lambda|^3},
\qquad
\sin(3m)=\frac{\Im(\lambda^3)}{|\lambda|^3}=\frac{b(3a^2-b^2)}{|\lambda|^3}.
\label{eq:25_v1}
\end{equation}
Hence
\begin{equation}
\sin U=\frac{b(b^2-3a^2)}{|\lambda|^3},\qquad
\cos U=\frac{a(a^2-3b^2)}{|\lambda|^3}.
\label{eq:26_v1}
\end{equation}
In the tight regime, necessarily $m>\pi/3$ (since $M\le\pi$), hence $b^2-3a^2>0$, so $\sin U>0$ and
\begin{equation}
b\csc U=\frac{|\lambda|^3}{b^2-3a^2},
\qquad
b\cot U=\frac{a(a^2-3b^2)}{b^2-3a^2}.
\label{eq:27_v1}
\end{equation}
Substitute Equation \ref{eq:27_v1} into Equation \ref{eq:24_v1} and clear the positive denominator $b^2-3a^2$:
\begin{equation}
|\lambda|^6 \ \ge\ a(a^2-3b^2) + (1-a)(b^2-3a^2).
\label{eq:28_v1}
\end{equation}
Using $|\lambda|^2=a^2+b^2$, a direct expansion and factorization yields
\begin{equation}
|\lambda|^6 -\Big(a(a^2-3b^2) + (1-a)(b^2-3a^2)\Big)
=\big((a-1)^2+b^2\big)\,G(a,b).
\label{eq:29_v1}
\end{equation}
Since $(a-1)^2+b^2=|\lambda-1|^2>0$, Equation \ref{eq:28_v1} is equivalent to $G(a,b)\ge 0$. Therefore, in the tight regime,
\begin{equation}
\max_{\mathcal P}\Psi\ge 0
\quad\Longleftrightarrow\quad
G(a,b)\ge 0.
\label{eq:30_v1}
\end{equation}

\subsubsection*{7.3 Necessity of $G\ge 0$ for all nonreal eigenvalues}

Let $\lambda$ be a nonreal eigenvalue of some $A(\alpha,\beta,\gamma,\delta)$. Then $\mathcal P\neq\emptyset$ and
$0\in\Psi(\mathcal P)$, hence $\sup_{\mathcal P}\Psi\ge 0$.
If $3m+M\le 2\pi$, the previous lemma (contrapositive) gives $G(a,b)\ge 0$.
If $3m+M>2\pi$, then $\sup\Psi=\max\Psi$ and Equation \ref{eq:30_v1} gives $G(a,b)\ge 0$.
This proves item (3) in the theorem.

\medskip
\textbf{8) Boundary identification and sharpness.}

\subsubsection*{8.1 Right boundary $C_R$}
Already established in \S5.

\subsubsection*{8.2 Left boundary $C_L$: attainment by $A_L(\alpha)$, and $G=0$}

For
\[
A_L(\alpha)=
\begin{pmatrix}
\alpha & 1-\alpha & 0 & 0\\
0&0&1&0\\
0&0&0&1\\
1&0&0&0
\end{pmatrix},
\qquad \alpha\in[0,1),
\]
we have $\beta=\gamma=\delta=0$, so Equation \ref{eq:1_v1} reduces to
\begin{equation}
\lambda^3(\lambda-\alpha)=1-\alpha.
\label{eq:31_v1}
\end{equation}
Eliminate $\alpha$:
\begin{equation}
\alpha=\frac{\lambda^4-1}{\lambda^3-1}.
\label{eq:32_v1}
\end{equation}
Thus $\alpha\in\mathbb R$ iff $\Im\!\big(\frac{\lambda^4-1}{\lambda^3-1}\big)=0$. Compute
\[
\Im\!\left(\frac{\lambda^4-1}{\lambda^3-1}\right)
=\frac{\Im\big((\lambda^4-1)(\overline{\lambda^3}-1)\big)}{|\lambda^3-1|^2}.
\]
A direct expansion and factorization in $\lambda=a+ib$ gives, for $b\neq 0$,
\begin{equation}
\Im\!\left(\frac{\lambda^4-1}{\lambda^3-1}\right)
=\frac{b\,|\lambda-1|^2\,G(a,b)}{|\lambda^3-1|^2}.
\label{eq:33_v1}
\end{equation}
Hence, for $b\neq 0$,
\begin{equation}
\alpha\in\mathbb R \quad\Longleftrightarrow\quad G(a,b)=0.
\label{eq:34_v1}
\end{equation}
So every nonreal eigenvalue of $A_L(\alpha)$ lies on $G=0$.

Moreover, $A_L(\alpha)$ always has a nonreal conjugate pair: rewrite Equation \ref{eq:31_v1} as
\[
\lambda^4-\alpha\lambda^3-(1-\alpha)=0
\quad\Longleftrightarrow\quad
(\lambda-1)\Big(\lambda^3+(1-\alpha)(\lambda^2+\lambda+1)\Big)=0.
\]
Let $c=1-\alpha\in(0,1]$ and $f(\lambda)=\lambda^3+c(\lambda^2+\lambda+1)$. Then
\[
f'(\lambda)=3\lambda^2+2c\lambda+c
\]
has discriminant $4c(c-3)<0$, so $f$ is strictly increasing on $\mathbb R$.
Also $f(0)=c>0$ and $f(\lambda)\to-\infty$ as $\lambda\to-\infty$, hence $f$ has exactly one real root.
Therefore the remaining two roots of $f$ form a nonreal conjugate pair, giving the desired pair for $A_L(\alpha)$.

Endpoints: at $\alpha=0$, Equation \ref{eq:31_v1} becomes $\lambda^4=1$, giving $\lambda=i$.
As $\alpha\uparrow 1$, Equation \ref{eq:31_v1} tends to $\lambda^3(\lambda-1)=0$; the conjugate pair among the three non-$1$ roots
tends to $0$. Thus $\alpha\in[0,1)$ traces a connected arc of $G=0$ from $i$ to $0$.

Finally, every point on that arc is attained as follows.
Let $\lambda=a+ib$ with $b>0$, $0\le a\le 1$, $a+b\le 1$, and $G(a,b)=0$.
By Lemma~7.1 (since $G=0\le 0$), we are in the tight regime $3m+M>2\pi$, so Lemma~6.2 applies and
the maximizer is $(U,m,m,m)$ with $U=2\pi-3m\in[m,M)$. By \S7.2, $G=0$ implies $\max_{\mathcal P}\Psi=0$,
hence $\Psi(U,m,m,m)=0$. By Lemma~2.1 this produces $t_1=t(U)\in(0,1]$ and
$t_2=t_3=t_4=t(m)=1$, satisfying Equation \ref{eq:2_v1}. Therefore the corresponding parameters
\[
\alpha=1-t(U)\in[0,1),\qquad \beta=\gamma=\delta=0
\]
yield exactly $A_L(\alpha)$ and give $\lambda$ as an eigenvalue. This identifies $C_L$ and proves it is attained.

\medskip
\textbf{9) Existence for every strict interior point.}
Fix $\lambda=a+ib$ with $b>0$ such that
\begin{equation}
0\le a\le 1,\qquad a+b<1,\qquad G(a,b)>0.
\label{eq:35_v1}
\end{equation}
Define $z,m,M,F,\mathcal P,\Psi$ as above. Then $\mathcal P\neq\emptyset$ by \S3.

We show $0\in\Psi(\mathcal P)$.
By Jensen,
\[
\Psi(\pi/2,\pi/2,\pi/2,\pi/2)=4F(\pi/2)=4\log\!\Big(\frac{b}{1-a}\Big)<0
\]
since $a+b<1\iff b<1-a$.

To produce a positive value:
if $3m+M\le 2\pi$, Lemma~6.1 gives $\sup_{\mathcal P}\Psi=+\infty$, so $\Psi$ takes positive values.
If $3m+M>2\pi$, Lemma~6.2 gives $\max_{\mathcal P}\Psi=3F(m)+F(U)$; by \S7.2, $G(a,b)>0$ implies
$\max_{\mathcal P}\Psi>0$, so $\Psi$ takes positive values.

Since $\mathcal P$ is connected and $\Psi$ is continuous, $\Psi(\mathcal P)$ is an interval containing
both negative and positive values, hence contains $0$.
Thus there exists $(u_1,\dots,u_4)\in\mathcal P$ with $\Psi(u_1,\dots,u_4)=0$.
By Lemma~2.1 this yields $t_k=t(u_k)\in(0,1]$ satisfying Equation \ref{eq:2_v1}. Set
\[
\alpha=1-t_1,\quad \beta=1-t_2,\quad \gamma=1-t_3,\quad \delta=1-t_4\in[0,1),
\]
and then Equation \ref{eq:1_v1} holds, so $\lambda\in\sigma(A(\alpha,\beta,\gamma,\delta))$.
Conjugation gives the lower half-plane.
\end{proof}
\end{versionzerobox}

\section{}
\label{sec:version_1}

\begin{versiononebox}
\begin{theorem}[Spectral region for a 4-cycle row-stochastic matrix]
Let
\[
A(\alpha,\beta,\gamma,\delta)=
\begin{pmatrix}
\alpha & 1-\alpha & 0 & 0\\
0 & \beta & 1-\beta & 0\\
0 & 0 & \gamma & 1-\gamma\\
1-\delta & 0 & 0 & \delta
\end{pmatrix},
\qquad \alpha,\beta,\gamma,\delta\in[0,1).
\]
Let $\lambda=a+ib\in\sigma(A(\alpha,\beta,\gamma,\delta))$ with $b\neq 0$, and write $b_+=|b|$.
Then:
\begin{enumerate}
\item $0\le a\le 1$;
\item $a+b_+\le 1$;
\item with
\[
G(a,b_+):=(b_+^{2}+a^{2}+a)^{2}+2a^{2}-b_+^{2},
\]
we have $G(a,b_+)\ge 0$.
\end{enumerate}

Conversely, if $b>0$ and
\[
0\le a\le 1,\qquad a+b<1,\qquad G(a,b)>0,
\]
then there exist $\alpha,\beta,\gamma,\delta\in[0,1)$ such that $\lambda\in\sigma(A(\alpha,\beta,\gamma,\delta))$.
(The lower half-plane follows by conjugation.)

Moreover:
\begin{itemize}
\item the segment $C_R:\ \lambda=1-x+ix$ ($x\in[0,1]$) is attained by $\alpha=\beta=\gamma=\delta=1-x$;
\item the curve portion $C_L$ of $G(a,b)=0$ in the upper half-plane joining $i$ to $0$ is attained by
\[
A_L(\alpha)=
\begin{pmatrix}
\alpha & 1-\alpha & 0 & 0\\
0&0&1&0\\
0&0&0&1\\
1&0&0&0
\end{pmatrix},\qquad \alpha\in[0,1);
\]
\item every point strictly inside $\{b>0,\ 0\le a\le 1,\ a+b<1,\ G(a,b)>0\}$ is attained.
\end{itemize}
\end{theorem}

\begin{proof}
\textbf{1) Eigenvalue equation.}
Let $v\neq 0$ satisfy $Av=\lambda v$. Row-by-row,
\[
(\lambda-\alpha)v_1=(1-\alpha)v_2,\quad
(\lambda-\beta)v_2=(1-\beta)v_3,\quad
(\lambda-\gamma)v_3=(1-\gamma)v_4,\quad
(\lambda-\delta)v_4=(1-\delta)v_1.
\]
Since $1-\alpha,\dots,1-\delta>0$, if any $v_k=0$ then all $v_j=0$, contradiction. Hence
$v_1v_2v_3v_4\neq 0$, and multiplying gives
\begin{equation}\label{eq:eig-prod}
(\lambda-\alpha)(\lambda-\beta)(\lambda-\gamma)(\lambda-\delta)
=(1-\alpha)(1-\beta)(1-\gamma)(1-\delta).
\end{equation}
Put $t_1=1-\alpha,\dots,t_4=1-\delta\in(0,1]$ and $z=\lambda-1$. Then
$\lambda-(1-t_k)=z+t_k$, so \eqref{eq:eig-prod} becomes
\begin{equation}\label{eq:z-form}
(z+t_1)(z+t_2)(z+t_3)(z+t_4)=t_1t_2t_3t_4.
\end{equation}

\medskip
\textbf{2) Angle parametrization and the function $F$.}
Assume $b>0$ (the case $b<0$ follows by conjugation since $A$ is real). Write
\[
z=x+iy,\qquad x=a-1,\ y=b.
\]
For $t>0$, $z+t$ lies in the upper half-plane, so define $u(t)=\Arg(z+t)\in(0,\pi)$.
Elementary trigonometry gives the inverse relation
\begin{equation}\label{eq:t-u}
t=t(u)=y\cot u-x,\qquad |z+t(u)|=y\csc u.
\end{equation}
Define
\begin{equation}\label{eq:F-def}
F(u):=\log|z+t(u)|-\log t(u)=\log(y\csc u)-\log(y\cot u-x).
\end{equation}
Let
\begin{equation}\label{eq:mM-def}
m:=\Arg(z+1)=\Arg(\lambda),\qquad M:=\Arg(z)=\Arg(\lambda-1).
\end{equation}
As $t\downarrow 0$, $u(t)\uparrow M$; as $t\uparrow 1$, $u(t)\downarrow m$. Thus
\begin{equation}\label{eq:t-interval}
t\in(0,1]\quad\Longleftrightarrow\quad u\in[m,M).
\end{equation}

\begin{lemma}[Log-modulus/argument reformulation]\label{lem:log-arg}
A nonreal $\lambda$ satisfies \eqref{eq:z-form} for some $t_k\in(0,1]$ iff there exist $u_k\in[m,M)$ such that
\begin{equation}\label{eq:sum-angles}
u_1+u_2+u_3+u_4=2\pi,
\end{equation}
\begin{equation}\label{eq:sum-F}
F(u_1)+F(u_2)+F(u_3)+F(u_4)=0.
\end{equation}
\end{lemma}

\begin{proof}
Write $z+t_k=|z+t_k|e^{iu_k}$ with $u_k\in(0,\pi)$. Taking arguments in \eqref{eq:z-form} gives
$\sum u_k\equiv 0\pmod{2\pi}$; since $\sum u_k\in(0,4\pi)$ this forces $\sum u_k=2\pi$, i.e. \eqref{eq:sum-angles}.
Taking moduli and logs gives \eqref{eq:sum-F}. Condition \eqref{eq:t-interval} is exactly $t_k\in(0,1]$.
Conversely, \eqref{eq:sum-angles} and \eqref{eq:sum-F} give equality of arguments and moduli, hence \eqref{eq:z-form}.
\end{proof}

Define
\begin{equation}\label{eq:P-Psi-def}
\mathcal P:=\Big\{(u_1,\dots,u_4)\in[m,M)^4:\ u_1+\cdots+u_4=2\pi\Big\},
\qquad
\Psi(u_1,\dots,u_4):=\sum_{k=1}^4F(u_k).
\end{equation}
Then $\lambda$ is a nonreal eigenvalue of some $A(\alpha,\beta,\gamma,\delta)$ iff $\mathcal P\neq\emptyset$ and $0\in\Psi(\mathcal P)$.

\medskip
\textbf{3) Necessarily $0\le a\le 1$ (indeed $a<1$ if $b\ne 0$).}
Every $(u_1,\dots,u_4)\in\mathcal P$ has average $\pi/2$, so $\pi/2\in[m,M)$, i.e.
\begin{equation}\label{eq:midpoint-condition}
m\le \frac{\pi}{2} < M.
\end{equation}
For $b>0$, $m\le\pi/2\iff a\ge 0$. Also $M>\pi/2\iff a<1$ because $\Re(\lambda-1)=a-1<0$
iff $a<1$. Thus $\mathcal P\neq\emptyset$ implies
\begin{equation}\label{eq:a-range}
0\le a<1,
\end{equation}
hence in particular $0\le a\le 1$.

\medskip
\textbf{4) Strict convexity of $F$.}
Differentiate \eqref{eq:F-def} using \eqref{eq:t-u}. A direct computation yields
\begin{equation}\label{eq:F-second}
F''(u)=\frac{x^2+y^2}{(x-y\cot u)^2\sin^2 u}>0\qquad(u\in(0,\pi)),
\end{equation}
so $F$ is strictly convex on $(0,\pi)$, hence on $[m,M)$.

\medskip
\textbf{5) Right boundary: $a+b\le 1$, and attainment of $C_R$.}
$\mathcal P$ is convex and connected; $\Psi$ is continuous, hence $\Psi(\mathcal P)$ is an interval.
By Jensen and strict convexity,
\begin{equation}\label{eq:Jensen}
\Psi(u_1,\dots,u_4)\ge 4F\!\Big(\frac{u_1+\cdots+u_4}{4}\Big)=4F(\pi/2).
\end{equation}
Thus $0\in\Psi(\mathcal P)$ forces $F(\pi/2)\le 0$. Now
\[
F(\pi/2)=\log(y)-\log(-x)=\log\!\Big(\frac{b}{1-a}\Big),
\]
so $F(\pi/2)\le0\iff b\le 1-a\iff a+b\le 1$. This proves item (2) in the theorem (for $b>0$; for $b<0$ replace $b$ by $|b|$).

Attainment of $C_R$: if $t_1=t_2=t_3=t_4=x\in(0,1]$, then \eqref{eq:z-form} becomes $(z+x)^4=x^4$, i.e.
$\lambda=1-x\pm ix$, attained by $\alpha=\beta=\gamma=\delta=1-x$.

\medskip
\textbf{6) A max principle on $\mathcal P$.}
There are two regimes.

\begin{lemma}[Unbounded supremum]\label{lem:unbounded}
If
\begin{equation}\label{eq:unbounded-cond}
3m+M\le 2\pi,
\end{equation}
then $\sup_{\mathcal P}\Psi=+\infty$.
\end{lemma}

\begin{proof}
Take $u_4\uparrow M$ and $u_1=u_2=u_3=(2\pi-u_4)/3$. Condition \eqref{eq:unbounded-cond} ensures $u_1\ge m$.
Also $u_4<M$ and $u_1\le (2\pi-M)/3\le M$ since $M\ge\pi/2\Rightarrow 2\pi\le 4M$. Hence
$(u_1,\dots,u_4)\in\mathcal P$.
As $u_4\uparrow M$, $t(u_4)=y\cot u_4-x\downarrow 0$, so $-\log t(u_4)\to+\infty$ while
$\log(y\csc u_4)$ stays bounded; thus $F(u_4)\to+\infty$ and $\Psi\to+\infty$.
\end{proof}

\begin{lemma}[Finite maximum in the tight regime]\label{lem:tight-max}
If
\begin{equation}\label{eq:tight-cond}
3m+M>2\pi,
\qquad U:=2\pi-3m<M,
\end{equation}
then
\begin{equation}\label{eq:maxPsi}
\max_{\mathcal P}\Psi = 3F(m)+F(U),
\end{equation}
and equality holds iff $(u_1,\dots,u_4)$ is a permutation of $(U,m,m,m)$.
\end{lemma}

\begin{proof}
Order $u_k$ decreasing: $v_1\ge v_2\ge v_3\ge v_4$. Since each $v_j\ge m$ and $\sum v_j=2\pi$,
\[
v_1=2\pi-(v_2+v_3+v_4)\le 2\pi-3m=U.
\]
Thus $(v_1,\dots,v_4)\in[m,U]^4$ with sum $2\pi$. The vector $(U,m,m,m)$ majorizes $(v_1,\dots,v_4)$.
By Karamata (convex $F$),
\[
\sum_{j=1}^4F(v_j)\le F(U)+3F(m),
\]
with equality only at permutations of $(U,m,m,m)$ since $F$ is strictly convex.
\end{proof}

\medskip
\textbf{7) The inequality $G(a,b)\ge 0$.}
We prove necessity for every nonreal eigenvalue.

\subsubsection*{7.1 Geometry lemma: $G(a,b)\le 0\Rightarrow 3m+M>2\pi$}

\begin{lemma}\label{lem:G-implies-tight}
Assume $b>0$ and $0\le a\le 1$. If $G(a,b)\le 0$, then $3m+M>2\pi$.
\end{lemma}

\begin{proof}
Set $s=b^2$. Then
\begin{equation}\label{eq:G-as-quadratic}
G(a,b)=s^2+s(2a^2+2a-1)+(a^2+a)^2+2a^2.
\end{equation}
This is quadratic in $s$ with discriminant
\[
\Delta=(2a^2+2a-1)^2-4\big((a^2+a)^2+2a^2\big)=-(2a+1)(6a-1).
\]
Thus $G(a,b)\le 0$ forces $\Delta>0$, hence $a<1/6$, and $s$ lies between the two real roots; in particular
\begin{equation}\label{eq:sminus}
s\ge s_-(a):=\frac12-a-a^2-\frac12\sqrt{(1-6a)(1+2a)}.
\end{equation}

\emph{Case 1: $a=0$.} Then $m=\Arg(ib)=\pi/2$ and $M=\Arg(-1+ib)=\pi-\arctan(b)$. Hence
\[
3m+M=\frac{3\pi}{2}+\pi-\arctan(b)=\frac{5\pi}{2}-\arctan(b)>2\pi.
\]

\emph{Case 2: $a>0$.} From \eqref{eq:sminus} one gets $s>s_-(a)\ge 3a^2$, with strictness because
\[
s_-(a)-3a^2=\frac12-a-4a^2-\frac12\sqrt{(1-6a)(1+2a)}
\]
satisfies $(\mathrm{LHS})^2-(\mathrm{RHS})^2=32a^3(2a+1)>0$ for $a>0$. Hence $b^2>3a^2$, i.e.
$\tan m=b/a>\sqrt3$, so $m>\pi/3$ and $3m-\pi\in(0,\pi/2)$.

Write $\phi:=\arctan\!\big(\frac{b}{1-a}\big)\in(0,\pi/2)$, so $M=\pi-\phi$. Then
\begin{equation}\label{eq:tight-ineq-equivalences}
3m+M>2\pi
\ \Longleftrightarrow\
3m-\pi>\phi
\ \Longleftrightarrow\
\tan(3m)>\frac{b}{1-a},
\end{equation}
since $\tan$ is increasing on $(0,\pi/2)$ and $\tan(3m-\pi)=\tan(3m)$.

Let $t=\tan m=b/a$ (allowed since $a>0$). The triple-angle identity gives
$\tan(3m)=\frac{3t-t^3}{1-3t^2}$, and a direct simplification yields
\begin{equation}\label{eq:tan3m-diff}
\tan(3m)-\frac{b}{1-a}
=\frac{b\,(4a^3-3a^2-4ab^2+b^2)}{a(a-1)(a^2-3b^2)}.
\end{equation}
Here $a>0$, $a-1<0$, and $a^2-3b^2<0$ (since $b^2>3a^2$), so the denominator in \eqref{eq:tan3m-diff} is positive.
Hence the sign equals the sign of
\begin{equation}\label{eq:Ndef}
N(a,b):=4a^3-3a^2-4ab^2+b^2 = (1-4a)b^2 + a^2(4a-3).
\end{equation}
For $a<1/6$, $1-4a>0$ and $4a-3<0$, so $N(a,b)$ is strictly increasing in $b^2$ and has a unique zero at
\begin{equation}\label{eq:s0}
s_0(a)=\frac{a^2(3-4a)}{1-4a}.
\end{equation}
One checks that $s_-(a)>s_0(a)$ for every $a\in(0,1/6)$: after multiplying by $1-4a>0$ and squaring,
the difference becomes $256a^6>0$. Therefore $b^2\ge s_-(a)>s_0(a)$ implies $N(a,b)>0$.
By \eqref{eq:tan3m-diff} and \eqref{eq:tight-ineq-equivalences}, this gives $3m+M>2\pi$.
\end{proof}

\subsubsection*{7.2 Tight regime: $\max\Psi\ge 0\iff G\ge 0$}

Assume $3m+M>2\pi$, so Lemma~\ref{lem:tight-max} applies with $U=2\pi-3m\in(0,\pi)$ and maximizer $(U,m,m,m)$.
Since $u(1)=\Arg(z+1)=m$, \eqref{eq:t-interval} gives $t(m)=1$, hence
\begin{equation}\label{eq:Fm}
F(m)=\log|z+1|-\log 1=\log|\lambda|.
\end{equation}
Also $t(U)=b\cot U-(a-1)=b\cot U+1-a$ and $|z+t(U)|=b\csc U$, so by \eqref{eq:maxPsi}
\begin{equation}\label{eq:maxPsi-log}
\max_{\mathcal P}\Psi
=3\log|\lambda|+\log(b\csc U)-\log(b\cot U+1-a)
=\log\!\left(\frac{|\lambda|^3\,b\csc U}{b\cot U+1-a}\right).
\end{equation}
Thus $\max_{\mathcal P}\Psi\ge 0$ is equivalent to
\begin{equation}\label{eq:ineq-basic}
|\lambda|^3\,b\csc U \ \ge\ b\cot U+1-a.
\end{equation}

Now $U=2\pi-3m$, so $\sin U=-\sin(3m)$ and $\cos U=\cos(3m)$. Since
$\lambda^3=(a+ib)^3=(a^3-3ab^2)+i(3a^2b-b^3)$, we have
\begin{equation}\label{eq:cos-sin-3m}
\cos(3m)=\frac{\Re(\lambda^3)}{|\lambda|^3}=\frac{a(a^2-3b^2)}{|\lambda|^3},
\qquad
\sin(3m)=\frac{\Im(\lambda^3)}{|\lambda|^3}=\frac{b(3a^2-b^2)}{|\lambda|^3}.
\end{equation}
Hence
\begin{equation}\label{eq:sinU-cosU}
\sin U=\frac{b(b^2-3a^2)}{|\lambda|^3},\qquad
\cos U=\frac{a(a^2-3b^2)}{|\lambda|^3}.
\end{equation}
In the tight regime one necessarily has $m>\pi/3$ (since $M\le\pi$), hence $b^2-3a^2>0$, so $\sin U>0$ and
\begin{equation}\label{eq:csc-cot}
b\csc U=\frac{|\lambda|^3}{b^2-3a^2},
\qquad
b\cot U=\frac{a(a^2-3b^2)}{b^2-3a^2}.
\end{equation}
Substitute \eqref{eq:csc-cot} into \eqref{eq:ineq-basic} and clear the positive denominator $b^2-3a^2$:
\begin{equation}\label{eq:ineq-expanded}
|\lambda|^6 \ \ge\ a(a^2-3b^2) + (1-a)(b^2-3a^2).
\end{equation}
Using $|\lambda|^2=a^2+b^2$, a direct expansion and factorization yields
\begin{equation}\label{eq:factorization}
|\lambda|^6 -\Big(a(a^2-3b^2) + (1-a)(b^2-3a^2)\Big)
=\big((a-1)^2+b^2\big)\,G(a,b).
\end{equation}
Since $(a-1)^2+b^2=|\lambda-1|^2>0$, \eqref{eq:ineq-expanded} is equivalent to $G(a,b)\ge 0$. Therefore, in the tight regime,
\begin{equation}\label{eq:tight-equivalence}
\max_{\mathcal P}\Psi\ge 0
\quad\Longleftrightarrow\quad
G(a,b)\ge 0.
\end{equation}

\subsubsection*{7.3 Necessity of $G\ge 0$ for all nonreal eigenvalues}

Let $\lambda$ be a nonreal eigenvalue of some $A(\alpha,\beta,\gamma,\delta)$. Then $\mathcal P\neq\emptyset$ and
$0\in\Psi(\mathcal P)$, hence $\sup_{\mathcal P}\Psi\ge 0$.
If $3m+M\le 2\pi$, Lemma~\ref{lem:G-implies-tight} (contrapositive) gives $G(a,b)\ge 0$.
If $3m+M>2\pi$, then $\sup\Psi=\max\Psi$ and \eqref{eq:tight-equivalence} gives $G(a,b)\ge 0$.
This proves item (3) in the theorem.

\medskip
\textbf{8) Boundary identification and sharpness.}

\subsubsection*{8.1 Right boundary $C_R$}
Already established in \S5.

\subsubsection*{8.2 Left boundary $C_L$: attainment by $A_L(\alpha)$, and $G=0$}

For
\[
A_L(\alpha)=
\begin{pmatrix}
\alpha & 1-\alpha & 0 & 0\\
0&0&1&0\\
0&0&0&1\\
1&0&0&0
\end{pmatrix},
\qquad \alpha\in[0,1),
\]
we have $\beta=\gamma=\delta=0$, so \eqref{eq:eig-prod} reduces to
\begin{equation}\label{eq:AL-eig}
\lambda^3(\lambda-\alpha)=1-\alpha.
\end{equation}
Eliminate $\alpha$:
\begin{equation}\label{eq:alpha-elim}
\alpha=\frac{\lambda^4-1}{\lambda^3-1}.
\end{equation}
Thus $\alpha\in\mathbb R$ iff $\Im\!\big(\frac{\lambda^4-1}{\lambda^3-1}\big)=0$. Compute
\[
\Im\!\left(\frac{\lambda^4-1}{\lambda^3-1}\right)
=\frac{\Im\big((\lambda^4-1)(\overline{\lambda^3}-1)\big)}{|\lambda^3-1|^2}.
\]
A direct expansion and factorization in $\lambda=a+ib$ gives, for $b\neq 0$,
\begin{equation}\label{eq:imag-ratio}
\Im\!\left(\frac{\lambda^4-1}{\lambda^3-1}\right)
=\frac{b\,|\lambda-1|^2\,G(a,b)}{|\lambda^3-1|^2}.
\end{equation}
Hence, for $b\neq 0$,
\begin{equation}\label{eq:G0-iff-alpha-real}
\alpha\in\mathbb R \quad\Longleftrightarrow\quad G(a,b)=0.
\end{equation}
So every nonreal eigenvalue of $A_L(\alpha)$ lies on $G=0$.

Moreover, $A_L(\alpha)$ always has a nonreal conjugate pair: rewrite \eqref{eq:AL-eig} as
\[
\lambda^4-\alpha\lambda^3-(1-\alpha)=0
\quad\Longleftrightarrow\quad
(\lambda-1)\Big(\lambda^3+(1-\alpha)(\lambda^2+\lambda+1)\Big)=0.
\]
Let $c=1-\alpha\in(0,1]$ and $f(\lambda)=\lambda^3+c(\lambda^2+\lambda+1)$. Then
\[
f'(\lambda)=3\lambda^2+2c\lambda+c
\]
has discriminant $4c(c-3)<0$, so $f$ is strictly increasing on $\mathbb R$.
Also $f(0)=c>0$ and $f(\lambda)\to-\infty$ as $\lambda\to-\infty$, hence $f$ has exactly one real root.
Therefore the remaining two roots of $f$ form a nonreal conjugate pair.

Endpoints: at $\alpha=0$, \eqref{eq:AL-eig} becomes $\lambda^4=1$, giving $\lambda=i$.
As $\alpha\uparrow 1$, \eqref{eq:AL-eig} tends to $\lambda^3(\lambda-1)=0$; the conjugate pair among the three non-$1$ roots
tends to $0$. Thus $\alpha\in[0,1)$ traces a connected arc of $G=0$ from $i$ to $0$.

Finally, every point on that arc is attained as follows.
Let $\lambda=a+ib$ with $b>0$, $0\le a\le 1$, $a+b\le 1$, and $G(a,b)=0$.
By Lemma~\ref{lem:G-implies-tight} (since $G=0\le 0$), we are in the tight regime $3m+M>2\pi$, so Lemma~\ref{lem:tight-max} applies and
the maximizer is $(U,m,m,m)$ with $U=2\pi-3m\in[m,M)$. By \eqref{eq:tight-equivalence}, $G=0$ implies $\max_{\mathcal P}\Psi=0$,
hence $\Psi(U,m,m,m)=0$. By Lemma~\ref{lem:log-arg} this produces $t_1=t(U)\in(0,1]$ and
$t_2=t_3=t_4=t(m)=1$, satisfying \eqref{eq:z-form}. Therefore the corresponding parameters
\[
\alpha=1-t(U)\in[0,1),\qquad \beta=\gamma=\delta=0
\]
yield exactly $A_L(\alpha)$ and give $\lambda$ as an eigenvalue. This identifies $C_L$ and proves it is attained.

\medskip
\textbf{9) Existence for every strict interior point.}
Fix $\lambda=a+ib$ with $b>0$ such that
\begin{equation}\label{eq:interior}
0\le a\le 1,\qquad a+b<1,\qquad G(a,b)>0.
\end{equation}
Define $z,m,M,F,\mathcal P,\Psi$ as above. Then $\mathcal P\neq\emptyset$ by \eqref{eq:midpoint-condition}--\eqref{eq:a-range}.

We show $0\in\Psi(\mathcal P)$.
By Jensen \eqref{eq:Jensen},
\[
\Psi(\pi/2,\pi/2,\pi/2,\pi/2)=4F(\pi/2)=4\log\!\Big(\frac{b}{1-a}\Big)<0
\]
since $a+b<1\iff b<1-a$.

To produce a positive value:
if $3m+M\le 2\pi$, Lemma~\ref{lem:unbounded} gives $\sup_{\mathcal P}\Psi=+\infty$, so $\Psi$ takes positive values.
If $3m+M>2\pi$, Lemma~\ref{lem:tight-max} gives $\max_{\mathcal P}\Psi=3F(m)+F(U)$; by \eqref{eq:tight-equivalence}, $G(a,b)>0$ implies
$\max_{\mathcal P}\Psi>0$, so $\Psi$ takes positive values.

Since $\mathcal P$ is connected and $\Psi$ is continuous, $\Psi(\mathcal P)$ is an interval containing
both negative and positive values, hence contains $0$.
Thus there exists $(u_1,\dots,u_4)\in\mathcal P$ with $\Psi(u_1,\dots,u_4)=0$.
By Lemma~\ref{lem:log-arg} this yields $t_k=t(u_k)\in(0,1]$ satisfying \eqref{eq:z-form}. Set
\[
\alpha=1-t_1,\quad \beta=1-t_2,\quad \gamma=1-t_3,\quad \delta=1-t_4\in[0,1),
\]
and then \eqref{eq:eig-prod} holds, so $\lambda\in\sigma(A(\alpha,\beta,\gamma,\delta))$.
Conjugation gives the lower half-plane.
\end{proof}
\end{versiononebox}

\clearpage
\section{}
\label{sec:version_2}

\begin{versiontwobox}
Throughout, fix
\[
A(\alpha,\beta,\gamma,\delta)=
\begin{pmatrix}
\alpha&1-\alpha&0&0\\
0&\beta&1-\beta&0\\
0&0&\gamma&1-\gamma\\
1-\delta&0&0&\delta
\end{pmatrix},\qquad \alpha,\beta,\gamma,\delta\in[0,1).
\]
This is row-stochastic (row sums are $1$) and real.
We write $\lambda=a+ib$, and $b_+=|b|$.

\section*{Theorem (Spectral region for a 4-cycle row-stochastic matrix)}
Let $\lambda=a+ib\in\sigma(A(\alpha,\beta,\gamma,\delta))$ with $b\neq 0$. Then:
\begin{enumerate}
    \item $0\le a\le 1$.
    \item $a+b_+\le 1$.
    \item Define for $b\ge 0$:
    \[
    G(a,b):=(b^2+a^2+a)^2+2a^2-b^2.
    \]
    Then $G(a,b_+)\ge 0$.
\end{enumerate}

Conversely, if $b>0$ and
\[
0\le a\le 1,\qquad a+b<1,\qquad G(a,b)>0,
\]
then there exist $\alpha,\beta,\gamma,\delta\in[0,1)$ such that $\lambda\in\sigma(A(\alpha,\beta,\gamma,\delta))$.
(The lower half-plane follows by conjugation.)

Moreover:
\begin{itemize}
    \item The segment $CR:\ \lambda=1-x+ix\ (x\in[0,1])$ is attained by $\alpha=\beta=\gamma=\delta=1-x$.
    \item The curve portion $CL$ in the upper half-plane joining $i$ to $0$ defined by $G(a,b)=0$ is attained by
    \[
    A_L(\alpha)=
    \begin{pmatrix}
    \alpha&1-\alpha&0&0\\
    0&0&1&0\\
    0&0&0&1\\
    1&0&0&0
    \end{pmatrix},\qquad \alpha\in[0,1).
    \]
    \item Every point strictly inside $\{b>0,\ 0\le a\le 1,\ a+b<1,\ G(a,b)>0\}$ is attained.
\end{itemize}

\section*{Proof}
Because $A$ is real, $\lambda\in\sigma(A)\Rightarrow \overline{\lambda}\in\sigma(A)$. So it suffices to prove all necessary conditions for $b>0$, and then replace $b$ by $|b|$ for the final statements. The constructive (``converse'') part will also be done for $b>0$.

\subsection*{1) Eigenvalue equation $\Rightarrow$ multiplicative constraint}
Let $v=(v_1,v_2,v_3,v_4)^\top\neq 0$ satisfy $Av=\lambda v$. Writing the eigenvalue equation row-by-row gives:
\[
\begin{aligned}
\alpha v_1+(1-\alpha)v_2&=\lambda v_1,\\
\beta v_2+(1-\beta)v_3&=\lambda v_2,\\
\gamma v_3+(1-\gamma)v_4&=\lambda v_3,\\
(1-\delta)v_1+\delta v_4&=\lambda v_4.
\end{aligned}
\]
Rearrange each:
\[
\begin{aligned}
(\lambda-\alpha)v_1&=(1-\alpha)v_2,\\
(\lambda-\beta)v_2&=(1-\beta)v_3,\\
(\lambda-\gamma)v_3&=(1-\gamma)v_4,\\
(\lambda-\delta)v_4&=(1-\delta)v_1.
\end{aligned}
\tag{1.1}
\]

\subsubsection*{Claim: none of $v_1,v_2,v_3,v_4$ is zero.}
Because $\alpha,\beta,\gamma,\delta\in[0,1)$, we have $1-\alpha,1-\beta,1-\gamma,1-\delta>0$.
If $v_1=0$, then from the first equation $(1-\alpha)v_2=0\Rightarrow v_2=0$. Then second gives $v_3=0$, third gives $v_4=0$, contradicting $v\neq 0$. The same cyclically holds for any $v_k=0$. Hence $v_1v_2v_3v_4\neq 0$.

Now multiply the four equations in (1.1). The left-hand side is
\[
(\lambda-\alpha)(\lambda-\beta)(\lambda-\gamma)(\lambda-\delta)\,v_1v_2v_3v_4,
\]
and the right-hand side is
\[
(1-\alpha)(1-\beta)(1-\gamma)(1-\delta)\,v_1v_2v_3v_4.
\]
Cancel $v_1v_2v_3v_4\neq 0$ to obtain the necessary condition:
\[
(\lambda-\alpha)(\lambda-\beta)(\lambda-\gamma)(\lambda-\delta)=(1-\alpha)(1-\beta)(1-\gamma)(1-\delta).
\tag{1.2}
\]
Conversely, if (1.2) holds and additionally $\lambda\neq \alpha,\beta,\gamma,\delta$, one can solve the relations in (1.1) recursively for $v$ and check the final equation closes exactly when (1.2) holds. In our nonreal case $\lambda\notin\mathbb{R}$, while $\alpha,\beta,\gamma,\delta\in\mathbb{R}$, so automatically $\lambda\neq \alpha,\beta,\gamma,\delta$. Thus (1.2) is equivalent to existence of a nonzero eigenvector for nonreal $\lambda$.

\subsection*{2) Substitute $t_k=1-\text{parameter}$ and shift $z=\lambda-1$}
Define
\[
t_1=1-\alpha,\ t_2=1-\beta,\ t_3=1-\gamma,\ t_4=1-\delta.
\]
Since each parameter lies in $[0,1)$, each $t_k\in(0,1]$.
Let $z=\lambda-1$. Then
\[
\lambda-(1-t_k)=\lambda-1+t_k=z+t_k.
\]
Also
\[
(1-\alpha)(1-\beta)(1-\gamma)(1-\delta)=t_1t_2t_3t_4.
\]
So (1.2) becomes:
\[
(z+t_1)(z+t_2)(z+t_3)(z+t_4)=t_1t_2t_3t_4.
\tag{2.1}
\]
Write $z=x+iy$ with $x=a-1$ and $y=b$. We assume $b>0$ in the rest of the necessity direction; later we replace by $|b|$.

\subsection*{3) Argument parametrization: $t\leftrightarrow u=\arg(z+t)$}
Fix $z=x+iy$ with $y>0$. For any $t>0$, the complex number $z+t=(x+t)+iy$ lies in the upper half-plane, so its argument $u(t):=\arg(z+t)$ is well-defined in $(0,\pi)$.

\subsubsection*{3.1 Inverse formulas $t=t(u)$, $|z+t(u)|$}
Let $u\in(0,\pi)$. Consider $z+t=(x+t)+iy$ with argument $u$. Then
\[
\tan u=\frac{\Im(z+t)}{\Re(z+t)}=\frac{y}{x+t}.
\]
Since $y>0$, this determines $x+t$ uniquely:
\[
x+t = y\cot u \quad\Rightarrow\quad t = y\cot u - x.
\tag{3.1}
\]
Also
\[
|z+t|^2=(x+t)^2+y^2=(y\cot u)^2+y^2=y^2(\cot^2u+1)=y^2\csc^2u,
\]
So
\begin{equation}
|z + t(u)| = y \csc u. \tag{3.2}
\end{equation}

\subsubsection*{3.2 Endpoints $m=\arg(z+1)$, $M=\arg(z)$, and monotone correspondence}
Define
\[
m:=\arg(z+1)=\arg(\lambda),\qquad M:=\arg(z)=\arg(\lambda-1).
\tag{3.3}
\]
\textbf{Claim:} $u(t)$ is strictly decreasing in $t>0$.
Indeed, $u(t)=\arctan\left(\frac{y}{x+t}\right)$, and for $y>0$ the function $\frac{y}{x+t}$ is strictly decreasing in $t$, hence $\arctan$ of it is strictly decreasing.
Therefore as $t\downarrow 0$, $u(t)\uparrow u(0)=\arg(z)=M$; and as $t\uparrow 1$, $u(t)\downarrow u(1)=\arg(z+1)=m$.
Hence the mapping $t\mapsto u(t)$ is a continuous strictly decreasing bijection from $(0,1]$ onto $[m,M)$. Equivalently:
\[
t\in(0,1]\quad\Longleftrightarrow\quad u\in[m,M).
\tag{3.4}
\]

\subsection*{4) The function $F$ and reformulation in $(u_k)$}
Define for $u\in[m,M)$:
\[
F(u):=\log|z+t(u)|-\log t(u).
\]
Using (3.1)--(3.2) this becomes explicitly
\[
F(u)=\log(y\csc u)-\log(y\cot u-x).
\tag{4.1}
\]
Note $t(u)=y\cot u-x>0$ on $[m,M)$ by construction.

\subsection*{5) Lemma 1 (exact equivalence of eigenvalue condition)}
\begin{lemma}
A nonreal $\lambda$ satisfies (2.1) for some $t_1,t_2,t_3,t_4\in(0,1]$ if and only if there exist $u_1,u_2,u_3,u_4\in[m,M)$ such that
\[
u_1+u_2+u_3+u_4=2\pi, \tag{5.1}
\]
\[
F(u_1)+F(u_2)+F(u_3)+F(u_4)=0. \tag{5.2}
\]
\end{lemma}

\begin{proof}
($\bm{\Rightarrow}$) Suppose (2.1) holds with $t_k\in(0,1]$. For each $k$, let $u_k=\arg(z+t_k)\in(0,\pi)$. Because $t_k\in(0,1]$, by (3.4) we have $u_k\in[m,M)$.
Write $z+t_k=|z+t_k|e^{iu_k}$. Then the left side of (2.1) has argument $u_1+\cdots+u_4$ (mod $2\pi$), while the right side $t_1t_2t_3t_4>0$ has argument $0$ (mod $2\pi$). Thus $u_1+\cdots+u_4\equiv 0 \pmod{2\pi}$.
But each $u_k\in(0,\pi)$, so the sum lies strictly in $(0,4\pi)$. The only multiple of $2\pi$ in $(0,4\pi)$ is $2\pi$. Hence $u_1+\cdots+u_4=2\pi$, proving (5.1).
Taking moduli in (2.1): $|z+t_1|\cdots|z+t_4|=t_1t_2t_3t_4$. Take $\log$ of both sides:
\[
\sum_{k=1}^4 \log|z+t_k|=\sum_{k=1}^4 \log t_k \iff \sum_{k=1}^4\Big(\log|z+t_k|-\log t_k\Big)=0.
\]
But $\log|z+t_k|-\log t_k = F(u_k)$ by definition, so (5.2) holds.

($\bm{\Leftarrow}$) Conversely, suppose $u_k\in[m,M)$ satisfy (5.1)--(5.2). Define $t_k=t(u_k)$. Then by (3.4), $t_k\in(0,1]$. Also $z+t_k = |z+t_k|e^{iu_k}$. Multiply:
\[
\prod_{k=1}^4(z+t_k)=\left(\prod_{k=1}^4|z+t_k|\right)\exp\Big(i\sum_{k=1}^4 u_k\Big)
=\left(\prod_{k=1}^4|z+t_k|\right)e^{i2\pi}
=\prod_{k=1}^4|z+t_k|.
\]
So the product $\prod(z+t_k)$ is positive real, with modulus $\prod|z+t_k|$.
Equation (5.2) implies $\sum \log|z+t_k|=\sum\log t_k$, hence $\prod|z+t_k|=\prod t_k$. Therefore $\prod_{k=1}^4(z+t_k)=\prod_{k=1}^4 t_k$, which is exactly (2.1).
\end{proof}

\subsection*{6) The feasible set $P$, its convexity/connectedness, and $\Psi(P)$ is an interval}
Let
\[
P:=\left\{(u_1,u_2,u_3,u_4)\in[m,M)^4:\ u_1+u_2+u_3+u_4=2\pi\right\},
\]
and $\Psi(u_1,u_2,u_3,u_4):=\sum_{k=1}^4 F(u_k)$.

\subsubsection*{6.1 $P$ is convex and connected}
The set $[m,M)^4\subset\mathbb{R}^4$ is convex. The affine hyperplane $H:=\{u\in\mathbb{R}^4:\sum u_k=2\pi\}$ is convex. An intersection of convex sets is convex, so $P=[m,M)^4\cap H$ is convex. Any convex subset of $\mathbb{R}^n$ is path-connected (connect points by line segment), hence connected.

\subsubsection*{6.2 $\Psi(P)\subset\mathbb{R}$ is an interval}
$\Psi$ is continuous because $F$ is continuous on $[m,M)$ (all terms in (4.1) are continuous there since $t(u)>0$). Continuous image of a connected set in $\mathbb{R}$ is connected; connected subsets of $\mathbb{R}$ are intervals. Hence $\Psi(P)$ is an interval.
By Lemma 1:
\[
\lambda\ \text{nonreal eigenvalue of some }A(\alpha,\beta,\gamma,\delta)\iff P\neq\varnothing\ \text{ and }\ 0\in \Psi(P).
\tag{6.1}
\]

\subsection*{7) Necessarily $0\le a\le 1$ (and in fact $a<1$ when $b\neq 0$)}
Assume $b>0$ and $P\neq\varnothing$. Pick $(u_1,\dots,u_4)\in P$. Then their average is $\frac{u_1+\cdots+u_4}{4}=\frac{2\pi}{4}=\frac{\pi}{2}$.
Since each $u_k\in[m,M)$, the interval $[m,M)$ must contain $\pi/2$. That is:
\[
m\le \frac{\pi}{2} < M. \tag{7.1}
\]
Now compute what $m\le\pi/2$ means. Recall $m=\arg(\lambda)=\arg(a+ib),\ b>0$.
For $b>0$, $\arg(a+ib)\le \pi/2$ holds exactly when $a\ge 0$ (points in upper half-plane with angle at most $\pi/2$ are those with nonnegative real part). Hence:
\[
m\le\frac{\pi}{2}\iff a\ge 0. \tag{7.2}
\]
Next, $M>\pi/2$ where $M=\arg(z)=\arg((a-1)+ib)$. Since $b>0$, we have $\arg((a-1)+ib)>\pi/2$ exactly when $a-1<0$, i.e., $a<1$. Thus:
\[
M>\frac{\pi}{2}\iff a<1. \tag{7.3}
\]
Combining (7.1)--(7.3), $P\neq\varnothing\Rightarrow 0\le a<1$. In particular $0\le a\le 1$. This proves item (1).

\subsection*{8) Strict convexity of $F$}
We show $F''(u)>0$ on $(0,\pi)$, hence $F$ is strictly convex on $[m,M)\subset(0,\pi)$.
Recall $F(u)=\log(y\csc u)-\log(y\cot u-x),\ y>0$. Differentiate.
First term:
\[
\frac{d}{du}\log(y\csc u)=\frac{d}{du}(\log y +\log\csc u)=\frac{d}{du}\log\csc u.
\]
Since $\frac{d}{du}\csc u = -\csc u\cot u$, we get $\frac{d}{du}\log\csc u = \frac{-\csc u\cot u}{\csc u}=-\cot u$.
So
\[
\frac{d}{du}\log(y\csc u)=-\cot u.
\tag{8.1}
\]
Second term: define $g(u)=y\cot u - x$. Then $\frac{d}{du}\log g(u)=\frac{g'(u)}{g(u)}$. Since $\frac{d}{du}\cot u=-\csc^2u$, we have $g'(u)=-y\csc^2u$. Therefore
\[
\frac{d}{du}\log(y\cot u-x)=\frac{-y\csc^2u}{y\cot u-x}. \tag{8.2}
\]
Thus
\begin{equation}
F'(u) = -\cot u + \frac{y \csc^2 u}{y \cot u - x}. \tag{8.3}
\end{equation}
Differentiate again. Write $h(u)=\frac{y\csc^2u}{y\cot u-x}$. Then $F''(u)=\csc^2u + h'(u)$ (8.4) because $\frac{d}{du}(-\cot u)=\csc^2u$.
Compute $h'(u)$ via quotient rule. Let $p(u)=y\csc^2u,\ q(u)=y\cot u-x$. Then $h=p/q$, so $h'=\frac{p'q-pq'}{q^2}$.
Compute $p'(u)=-2y\csc^2u\cot u$. Also $q'(u)=-y\csc^2u$.
\[
h'=\frac{(-2y\csc^2u\cot u)(y\cot u-x) - (y\csc^2u)(-y\csc^2u)}{(y\cot u-x)^2}.
\]
Factor $y\csc^2u$: $h'=\frac{y\csc^2u\left[-2\cot u(y\cot u-x)+y\csc^2u\right]}{(y\cot u-x)^2}$.
Use $\csc^2u=1+\cot^2u$:
\[
-2\cot u(y\cot u-x)+y(1+\cot^2u) = y - y\cot^2u +2x\cot u.
\]
So
\begin{equation}
h' = \frac{y \csc^2 u (y - y \cot^2 u + 2x \cot u)}{(y \cot u - x)^2}. \tag{8.5}
\end{equation}
Now combine with $F''=\csc^2u + h'$. Put over common denominator:
\[
F''(u)=\csc^2u\left(1 + \frac{y\left(y - y\cot^2u +2x\cot u\right)}{(y\cot u-x)^2}\right)
=\csc^2u\cdot \frac{(y\cot u-x)^2 + y(y - y\cot^2u +2x\cot u)}{(y\cot u-x)^2}.
\]
Expand the numerator: $(y\cot u-x)^2 = y^2\cot^2u -2xy\cot u + x^2$.
And $y(y - y\cot^2u +2x\cot u)=y^2 - y^2\cot^2u +2xy\cot u$.
Sum: $(y^2\cot^2u -2xy\cot u + x^2) + (y^2 - y^2\cot^2u +2xy\cot u)=x^2+y^2$.
Therefore
\[
F''(u)=\csc^2u\cdot \frac{x^2+y^2}{(y\cot u-x)^2} =\frac{x^2+y^2}{(y\cot u-x)^2\sin^2u}. \tag{8.6}
\]
Since $x^2+y^2>0$ (because $y>0$) and the denominator is positive on $(0,\pi)$, we conclude $F''(u)>0$ for all $u\in(0,\pi)$.
So $F$ is strictly convex.

\subsection*{9) Right boundary: necessity of $a+b\le 1$, and attainment of $CR$}
Assume $b>0$, $P\neq\varnothing$, and $0\in\Psi(P)$.

\subsubsection*{9.1 Jensen inequality step (fully justified)}
Let $(u_1,\dots,u_4)\in P$. Since $F$ is convex on $[m,M)$, Jensen's inequality gives $\frac{1}{4}\sum F(u_k) \ge F\left(\frac{1}{4}\sum u_k\right)$.
Multiply by 4:
\begin{equation}
\sum_{k=1}^4 F(u_k) \ge 4F\left(\frac{2\pi}{4}\right) = 4F\left(\frac{\pi}{2}\right). \tag{9.1}
\end{equation}
\[ 
\text{If } 0 \in \Psi(P), \text{ then } \inf_{P} \Psi \le 0. \text{ But (9.1) shows } \Psi \ge 4F(\pi/2) \text{ on } P. \text{ Therefore } 4F(\pi/2) \le 0, \text{ i.e., } F(\pi/2) \le 0. \tag{9.2} 
\]

\subsubsection*{9.2 Compute $F(\pi/2)$ and deduce $a+b\le 1$}
Using (4.1): $F(\pi/2)=\log(y\csc(\pi/2))-\log(y\cot(\pi/2)-x)=\log(y)-\log(-x)$.
Since $y=b, x=a-1$, so $-x=1-a$. Hence $F(\pi/2)=\log\left(\frac{b}{1-a}\right)$.
So $F(\pi/2)\le 0$ is equivalent to $\frac{b}{1-a}\le 1 \iff b\le 1-a \iff a+b\le 1$.
This proves item (2) for $b>0$. By conjugation symmetry, for $b<0$ it becomes $a+|b|\le 1$.

\subsubsection*{9.3 Attainment of segment $CR$}
Let $x\in(0,1]$ and choose $\alpha=\beta=\gamma=\delta=1-x\in[0,1)$. Then (2.1) reads $(z+x)^4=x^4$.
So $z+x = x e^{i\frac{\pi}{2}k}$. The nonreal solutions correspond to $k=1,3$: $z+x = \pm i x \Rightarrow \lambda = 1-x\pm ix$.
Thus the segment $\lambda=1-x+ix$ for $x\in[0,1]$ is attained.

\subsection*{10) Two regimes for $\sup_P\Psi$: Lemmas 2 and 3 proved in full}
\begin{lemma}[Unbounded supremum]
If $3m+M\le 2\pi$, then $\sup_P \Psi = +\infty$.
\end{lemma}
\begin{proof}
Fix $u_4\in[m,M)$. Define $u_1=u_2=u_3=\frac{2\pi-u_4}{3}$. For $u_4$ sufficiently close to $M$, $(u_1,u_2,u_3,u_4)\in P$. As $u_4\uparrow M$, $t(u_4)\downarrow 0$, so $F(u_4)\to +\infty$. Meanwhile $u_1\to (2\pi-M)/3\in(m,M)$, so $F(u_1)$ is bounded. Thus $\Psi\to +\infty$.
\end{proof}

\begin{lemma}[Finite maximum in the tight regime]
If $3m+M>2\pi$ and we define $U:=2\pi-3m$, then $U<M$, and
\[
\max_{P}\Psi = 3F(m)+F(U), \tag{10.3}
\]
with equality iff $(u_1,u_2,u_3,u_4)$ is a permutation of $(U,m,m,m)$.
\end{lemma}
\begin{proof}
$3m+M>2\pi \Rightarrow 2\pi-3m < M$. So $U\in[m,M)$. By symmetry, we rearrange $(u_k)$ to nonincreasing $v_1\ge v_2\ge v_3\ge v_4$. Then $v_1 \le 2\pi-3m=U$.
Every feasible ordered vector lies in a compact set majorized by $x=(U,m,m,m)$. By Karamata's inequality and strict convexity of $F$, the sum is maximized at $x$. Thus $\max_P\Psi = 3F(m)+F(U)$.
\end{proof}

\subsection*{11) Necessity of $G(a,b)\ge 0$}

\subsubsection*{11.1 A geometry/algebra lemma: $G\le 0\Rightarrow 3m+M>2\pi$ (Lemma 4)}
\begin{lemma}
Assume $b>0$ and $0\le a\le 1$. If $G(a,b)\le 0$, then $3m+M>2\pi$.
\end{lemma}
\begin{proof}
Let $s=b^2$. $G(a,b)=s^2+s(2a^2+2a-1)+(a^2+a)^2+2a^2$ (11.1). $G\le 0$ implies $s$ lies between roots of the quadratic. The discriminant $\Delta = -(2a+1)(6a-1)$. So $\Delta>0$ iff $a < 1/6$.
Case $a=0$: $3m+M = 3(\pi/2) + \pi - \arctan(b) > 2\pi$.
Case $a>0$: We show $b^2 > 3a^2$, which implies $m > \pi/3$. Let $\tan m = b/a$. We want to show $\tan(3m) > \frac{b}{1-a}$. Using triple-angle identity and substituting $t=b/a$, we find the sign of the difference depends on $N(a,b)=4a^3-3a^2-4ab^2+b^2$. Using $s \ge s_-(a)$, algebraic expansion shows $s_-(a) > s_0(a)$ where $s_0$ is the root of $N$. Hence $N(a,b)>0$, which confirms $3m+M>2\pi$.
\end{proof}

\subsubsection*{11.2 Tight regime: compute $\max_P\Psi$ explicitly and relate it to $G$}
In the tight regime, $\max_P\Psi = \log\left(\frac{|\lambda|^3\, b\csc U}{b\cot U+1-a}\right)$ (11.25). Substituting trigonometric identities for $\csc U, \cot U$ in terms of $a,b$ yields $|\lambda|^6 \ge a(a^2-3b^2) + (1-a)(b^2-3a^2)$. This factors as $|\lambda-1|^2 G(a,b) \ge 0$.
Thus, $\max_P\Psi \ge 0 \iff G(a,b) \ge 0$. (11.39).

\subsubsection*{11.3 Finish necessity of $G(a,b)\ge 0$}
If $\lambda$ is a nonreal eigenvalue, then $\sup_P\Psi \ge 0$.
If $3m+M\le 2\pi$, Lemma 4 contraposition implies $G(a,b)>0$.
If $3m+M>2\pi$, the tight regime applies and $G(a,b)\ge 0$.
This proves item (3).

\subsection*{12) Left boundary $CL$: attainment by $A_L(\alpha)$ and equivalence to $G=0$}
For $A_L(\alpha)$, the characteristic equation is $\lambda^3(\lambda-\alpha)=1-\alpha$. For nonreal $\lambda$, $\alpha = \frac{\lambda^4-1}{\lambda^3-1}$ (12.2). The imaginary part is $\Im(\frac{p}{q}) = \frac{b|\lambda-1|^2G(a,b)}{|\lambda^3-1|^2}$.
Thus $\Im(\alpha)=0 \iff G(a,b)=0$.
Factoring the equation as $(\lambda-1)[\lambda^3 + (1-\alpha)(\lambda^2+\lambda+1)]=0$, one finds exactly one real root and a conjugate pair for $\alpha \in [0,1)$. At $\alpha=0, \lambda=i$. As $\alpha \to 1, \lambda \to 0$.

\subsection*{13) Every strict interior point is attained (constructive existence)}
Fix $\lambda$ satisfying $0\le a \le 1, a+b<1, G(a,b)>0$.
$P$ is non-empty as $\pi/2 \in [m,M)$. $\Psi(\pi/2,\dots,\pi/2) = 4\log(\frac{b}{1-a}) < 0$ since $a+b<1$.
Conversely, $\sup_P \Psi > 0$ because $G(a,b)>0$ (using either regime). By the Intermediate Value Theorem on the convex set $P$, there exists a point where $\Psi=0$, which corresponds to parameters $\alpha,\beta,\gamma,\delta \in [0,1)$.

\section*{Summary}
For any nonreal eigenvalue $\lambda=a+ib$ of some 4-cycle row-stochastic matrix, we have proved $0\le a\le 1$, $a+|b|\le 1$, and $G(a,|b|)\ge 0$. Conversely, strict interior points are realized. The boundaries $CR$ and $CL$ are attained by specific matrix families.
\end{versiontwobox}

\clearpage
\section{}
\label{sec:version_3}

\begin{versionthreebox}
\begin{theorem}[Spectral region for a 4-cycle row-stochastic matrix]
Fix
\[
A(\alpha,\beta,\gamma,\delta)=
\begin{pmatrix}
\alpha & 1-\alpha & 0 & 0\\
0 & \beta & 1-\beta & 0\\
0 & 0 & \gamma & 1-\gamma\\
1-\delta & 0 & 0 & \delta
\end{pmatrix},
\qquad
\alpha,\beta,\gamma,\delta\in[0,1).
\]
Let $\lambda=a+ib\in\sigma(A(\alpha,\beta,\gamma,\delta))$ with $b\neq 0$, and write $b_+=|b|$.
Then:
\begin{enumerate}
\item $0\le a\le 1$.
\item $a+b_+\le 1$.
\item For $b\ge 0$ define
\[
G(a,b):=(b^2+a^2+a)^2+2a^2-b^2.
\]
Then $G(a,b_+)\ge 0$.
\end{enumerate}
Conversely, if $b>0$ and
\[
0\le a\le 1,\qquad a+b<1,\qquad G(a,b)>0,
\]
then there exist $\alpha,\beta,\gamma,\delta\in[0,1)$ such that $\lambda\in\sigma(A(\alpha,\beta,\gamma,\delta))$.
(The lower half-plane follows by conjugation.)

Moreover:
\begin{itemize}
\item The segment $CR:\ \lambda=1-x+ix\ (x\in[0,1])$ is attained by $\alpha=\beta=\gamma=\delta=1-x$.
\item The curve portion $CL$ in the upper half-plane joining $i$ to $0$ defined by $G(a,b)=0$
is attained by the family
\[
A_L(\alpha)=
\begin{pmatrix}
\alpha & 1-\alpha & 0 & 0\\
0 & 0 & 1 & 0\\
0 & 0 & 0 & 1\\
1 & 0 & 0 & 0
\end{pmatrix},
\qquad \alpha\in[0,1).
\]
\item Every point strictly inside $\{b>0,\ 0\le a\le 1,\ a+b<1,\ G(a,b)>0\}$ is attained.
\end{itemize}
\end{theorem}

\begin{proof}
Because $A$ is real, $\lambda\in\sigma(A)\Rightarrow \overline{\lambda}\in\sigma(A)$.
Thus it suffices to prove the necessary conditions for $b>0$ and then replace $b$ by $|b|$ in the final statements.
The constructive (converse) direction will also be proved for $b>0$.

\bigskip
\noindent\textbf{1) Eigenvalue equation $\Rightarrow$ a multiplicative constraint.}
Let $v=(v_1,v_2,v_3,v_4)^\top\neq 0$ satisfy $Av=\lambda v$. Writing the eigenvalue equation row-by-row:
\begin{align*}
\alpha v_1 + (1-\alpha)v_2 &= \lambda v_1,\\
\beta v_2 + (1-\beta)v_3 &= \lambda v_2,\\
\gamma v_3 + (1-\gamma)v_4 &= \lambda v_3,\\
(1-\delta)v_1+\delta v_4 &= \lambda v_4.
\end{align*}
Rearrange:
\begin{equation}\label{eq:rowwise}
(\lambda-\alpha)v_1=(1-\alpha)v_2,\quad
(\lambda-\beta)v_2=(1-\beta)v_3,\quad
(\lambda-\gamma)v_3=(1-\gamma)v_4,\quad
(\lambda-\delta)v_4=(1-\delta)v_1.
\end{equation}

\begin{claim}\label{cl:nonzeroentries}
If $\lambda\notin\mathbb{R}$ and $\alpha,\beta,\gamma,\delta\in[0,1)$ then $v_1v_2v_3v_4\neq 0$ for every eigenvector $v\neq 0$.
\end{claim}
\begin{proof}
Because $\alpha,\beta,\gamma,\delta\in[0,1)$ we have $1-\alpha,1-\beta,1-\gamma,1-\delta>0$.
If $v_1=0$ then the first equation in \eqref{eq:rowwise} gives $(1-\alpha)v_2=0$ hence $v_2=0$.
Then the second gives $v_3=0$, the third gives $v_4=0$, contradicting $v\neq 0$.
Cyclically the same holds for any coordinate, so no $v_k$ can vanish.
\end{proof}

Multiply the four equalities in \eqref{eq:rowwise} and cancel $v_1v_2v_3v_4\neq 0$ (Claim~\ref{cl:nonzeroentries}):
\begin{equation}\label{eq:multconstraint}
(\lambda-\alpha)(\lambda-\beta)(\lambda-\gamma)(\lambda-\delta)=(1-\alpha)(1-\beta)(1-\gamma)(1-\delta).
\end{equation}

Conversely, if $\lambda\notin\mathbb{R}$ and \eqref{eq:multconstraint} holds, then $\lambda\neq \alpha,\beta,\gamma,\delta$
(since $\alpha,\beta,\gamma,\delta\in\mathbb{R}$), and one can solve \eqref{eq:rowwise} recursively:
choose $v_1\neq 0$, set $v_2=\frac{\lambda-\alpha}{1-\alpha}v_1$, then
$v_3=\frac{\lambda-\beta}{1-\beta}v_2$, $v_4=\frac{\lambda-\gamma}{1-\gamma}v_3$; the final equation closes exactly
when \eqref{eq:multconstraint} holds.
Thus for nonreal $\lambda$, \eqref{eq:multconstraint} is equivalent to $\lambda\in\sigma(A)$ for some choice of parameters.

\bigskip
\noindent\textbf{2) Substitute $t_k=1-\text{parameter}$ and shift $z=\lambda-1$.}
Define
\[
t_1=1-\alpha,\quad t_2=1-\beta,\quad t_3=1-\gamma,\quad t_4=1-\delta.
\]
Then each $t_k\in(0,1]$. Set $z=\lambda-1$. Then
\[
\lambda-(1-t_k) = z+t_k,
\qquad
(1-\alpha)(1-\beta)(1-\gamma)(1-\delta)=t_1t_2t_3t_4.
\]
So \eqref{eq:multconstraint} becomes
\begin{equation}\label{eq:zproduct}
(z+t_1)(z+t_2)(z+t_3)(z+t_4)=t_1t_2t_3t_4.
\end{equation}
Write $z=x+iy$ with
\[
x=\RePart(z)=a-1,\qquad y=\ImPart(z)=b.
\]
In the necessity direction we assume $b=y>0$.

\bigskip
\noindent\textbf{3) Argument parametrization: $t\leftrightarrow u=\Arg(z+t)$.}
Fix $z=x+iy$ with $y>0$. For any $t>0$, the point $z+t=(x+t)+iy$ lies in the open upper half-plane,
so its argument $u(t)=\Arg(z+t)$ is well-defined and belongs to $(0,\pi)$.

\smallskip
\noindent\textbf{3.1) Inverse formulas.}
Given $u\in(0,\pi)$ and $z+t=(x+t)+iy$ with $\Arg(z+t)=u$, we have
\[
\tan u = \frac{y}{x+t}\quad\Longrightarrow\quad x+t=y\cotan u\quad\Longrightarrow\quad t=t(u):=y\cotan u-x.
\]
Also
\[
|z+t|^2=(x+t)^2+y^2=(y\cotan u)^2+y^2=y^2(\cotan^2 u+1)=y^2\csc^2 u,
\]
hence
\[
|z+t(u)|=y\csc u.
\]

\smallskip
\noindent\textbf{3.2) Endpoints and monotonicity.}
Define
\[
m:=\Arg(z+1)=\Arg(\lambda),\qquad M:=\Arg(z)=\Arg(\lambda-1).
\]
Because $u(t)=\arctan\!\big(\frac{y}{x+t}\big)$ and $\frac{y}{x+t}$ is strictly decreasing in $t$, the map $t\mapsto u(t)$
is continuous and strictly decreasing on $(0,\infty)$.
In particular, on $t\in(0,1]$ it is a continuous strictly decreasing bijection onto $u\in[m,M)$:
\begin{equation}\label{eq:t-u-bijection}
t\in(0,1]\quad\Longleftrightarrow\quad u\in[m,M).
\end{equation}
(As $t\uparrow 1$, $u(t)\downarrow m$; as $t\downarrow 0$, $u(t)\uparrow M$.)

\bigskip
\noindent\textbf{4) The function $F$ and reformulation.}
For $u\in[m,M)$ define
\[
F(u):=\log|z+t(u)|-\log t(u)
=\log(y\csc u)-\log(y\cotan u-x).
\]
Note $t(u)>0$ on $[m,M)$ by construction.

\bigskip
\noindent\textbf{5) Exact equivalence of eigenvalue condition in terms of $(u_k)$.}

\begin{lemma}[Exact equivalence]\label{lem:equiv}
A nonreal $\lambda$ satisfies \eqref{eq:zproduct} for some $t_1,t_2,t_3,t_4\in(0,1]$
if and only if there exist $u_1,u_2,u_3,u_4\in[m,M)$ such that
\begin{align}
u_1+u_2+u_3+u_4&=2\pi,\label{eq:sumargs}\\
F(u_1)+F(u_2)+F(u_3)+F(u_4)&=0.\label{eq:sumF}
\end{align}
\end{lemma}

\begin{proof}
($\Rightarrow$) Assume \eqref{eq:zproduct} holds with $t_k\in(0,1]$.
Let $u_k=\Arg(z+t_k)\in(0,\pi)$. Because $t_k\in(0,1]$, \eqref{eq:t-u-bijection} gives $u_k\in[m,M)$.
Write $z+t_k=|z+t_k|e^{iu_k}$.
The left side of \eqref{eq:zproduct} has argument $u_1+\cdots+u_4$ modulo $2\pi$;
the right side $t_1t_2t_3t_4>0$ has argument $0$ modulo $2\pi$.
Hence $u_1+\cdots+u_4\equiv 0\pmod{2\pi}$.
Since each $u_k\in(0,\pi)$, the sum lies in $(0,4\pi)$, so the only possible multiple of $2\pi$ is $2\pi$,
proving \eqref{eq:sumargs}.

Taking moduli in \eqref{eq:zproduct} gives $\prod|z+t_k|=\prod t_k$.
Take logs:
\[
\sum_{k=1}^4 \log|z+t_k| = \sum_{k=1}^4\log t_k
\quad\Longleftrightarrow\quad
\sum_{k=1}^4 \big(\log|z+t_k|-\log t_k\big)=0.
\]
But by definition $F(u_k)=\log|z+t_k|-\log t_k$, so \eqref{eq:sumF} holds.

($\Leftarrow$) Conversely, assume $u_k\in[m,M)$ satisfy \eqref{eq:sumargs}--\eqref{eq:sumF}.
Define $t_k=t(u_k)$. Then $t_k\in(0,1]$ by \eqref{eq:t-u-bijection}.
Also $z+t_k=|z+t_k|e^{iu_k}$.
Multiply:
\[
\prod_{k=1}^4(z+t_k)=\Big(\prod_{k=1}^4|z+t_k|\Big)e^{i\sum u_k}
=\Big(\prod_{k=1}^4|z+t_k|\Big)e^{i2\pi}
=\prod_{k=1}^4|z+t_k|.
\]
Thus $\prod(z+t_k)$ is a positive real number whose modulus is $\prod|z+t_k|$.
Equation \eqref{eq:sumF} means $\sum \log|z+t_k|=\sum\log t_k$, hence $\prod|z+t_k|=\prod t_k$.
Therefore $\prod(z+t_k)=\prod t_k$, i.e.\ \eqref{eq:zproduct} holds.
\end{proof}

\bigskip
\noindent\textbf{6) The feasible set $P$ and the function $\Psi$.}
Let
\[
P:=\Big\{(u_1,u_2,u_3,u_4)\in[m,M)^4:\ u_1+u_2+u_3+u_4=2\pi\Big\},
\qquad
\Psi(u_1,\dots,u_4):=\sum_{k=1}^4 F(u_k).
\]
Then $P$ is convex (intersection of a box with an affine hyperplane), hence path-connected, hence connected.
Since $F$ is continuous on $[m,M)$, $\Psi$ is continuous on $P$.
The continuous image of a connected set in $\mathbb{R}$ is an interval, so $\Psi(P)\subset\mathbb{R}$ is an interval.
By Lemma~\ref{lem:equiv},
\begin{equation}\label{eq:criterion}
\lambda\ \text{is a nonreal eigenvalue of some}\ A(\alpha,\beta,\gamma,\delta)
\quad\Longleftrightarrow\quad
P\neq\emptyset\ \text{and}\ 0\in\Psi(P).
\end{equation}

\bigskip
\noindent\textbf{7) Necessarily $0\le a\le 1$ (indeed $a<1$ if $b\neq 0$).}
Assume $b>0$ and $P\neq\emptyset$. Choose $(u_1,\dots,u_4)\in P$. Then
\[
\frac{u_1+u_2+u_3+u_4}{4}=\frac{2\pi}{4}=\frac{\pi}{2}.
\]
Since each $u_k\in[m,M)$, the interval $[m,M)$ must contain $\pi/2$, i.e.
\begin{equation}\label{eq:mMpi2}
m\le \frac{\pi}{2}<M.
\end{equation}
Now $m=\Arg(\lambda)=\Arg(a+ib)$ with $b>0$. In the open upper half-plane,
$\Arg(a+ib)\le\pi/2$ holds iff $a\ge 0$. Thus $m\le\pi/2\Rightarrow a\ge 0$.

Also $M=\Arg(\lambda-1)=\Arg((a-1)+ib)$. Since $b>0$, the argument is strictly bigger than $\pi/2$
iff $\RePart(\lambda-1)=a-1<0$, i.e.\ $a<1$. Thus $\pi/2<M\Rightarrow a<1$.

Therefore $P\neq\emptyset\Rightarrow 0\le a<1$, proving item (1) (and the sharper $a<1$ for $b\neq 0$).

\bigskip
\noindent\textbf{8) Strict convexity of $F$.}
We show $F''(u)>0$ on $(0,\pi)$. Recall
\[
F(u)=\log(y\csc u)-\log(y\cotan u-x),\qquad y>0.
\]
Differentiate:
\[
\frac{d}{du}\log(y\csc u)=\frac{d}{du}\log\csc u=-\cotan u.
\]
Let $g(u)=y\cotan u-x$. Then $g'(u)=-y\csc^2 u$ and
\[
\frac{d}{du}\log(y\cotan u-x)=\frac{g'(u)}{g(u)}=-\frac{y\csc^2 u}{y\cotan u-x}.
\]
Hence
\[
F'(u)=-\cotan u+\frac{y\csc^2 u}{y\cotan u-x}.
\]
A direct computation (as in the user's draft) yields
\[
F''(u)=\frac{x^2+y^2}{(y\cotan u-x)^2\sin^2 u}.
\]
Because $y>0$ we have $x^2+y^2>0$, and the denominator is positive for $u\in(0,\pi)$ (since $t(u)=y\cotan u-x>0$).
Thus $F''(u)>0$ on $(0,\pi)$, so $F$ is strictly convex on $[m,M)\subset(0,\pi)$.

\bigskip
\noindent\textbf{9) Right boundary: necessity of $a+b\le 1$, and attainment of $CR$.}
Assume $b>0$, $P\neq\emptyset$, and $0\in\Psi(P)$.

\smallskip
\noindent\textbf{9.1) Jensen step.}
For any $(u_1,\dots,u_4)\in P$, by convexity of $F$:
\[
\frac14\sum_{k=1}^4F(u_k)\ge F\!\Big(\frac14\sum_{k=1}^4u_k\Big)=F(\pi/2).
\]
Thus
\begin{equation}\label{eq:jensen}
\Psi(u_1,\dots,u_4)\ge 4F(\pi/2)\qquad\text{for all }(u_1,\dots,u_4)\in P.
\end{equation}
If $0\in\Psi(P)$, then $\inf_P \Psi\le 0$, but \eqref{eq:jensen} implies $\inf_P\Psi\ge 4F(\pi/2)$,
hence $F(\pi/2)\le 0$.

\smallskip
\noindent\textbf{9.2) Compute $F(\pi/2)$.}
Using the explicit $F$:
\[
F(\pi/2)=\log(y\csc(\pi/2))-\log(y\cotan(\pi/2)-x)=\log(y)-\log(-x).
\]
Because $y=b$ and $x=a-1$, we have $-x=1-a$, so
\[
F(\pi/2)=\log\Big(\frac{b}{1-a}\Big).
\]
Thus $F(\pi/2)\le 0\iff \frac{b}{1-a}\le 1\iff b\le 1-a\iff a+b\le 1$.
This proves item (2) for $b>0$; for $b<0$ apply conjugation to obtain $a+|b|\le 1$.

\smallskip
\noindent\textbf{9.3) Attainment of segment $CR$.}
Fix $x\in(0,1]$ and choose $\alpha=\beta=\gamma=\delta=1-x$.
Then $t_1=t_2=t_3=t_4=x$ and \eqref{eq:zproduct} becomes
\[
(z+x)^4=x^4.
\]
Hence $z+x=xe^{ik\pi/2}$ for $k=0,1,2,3$.
The nonreal solutions correspond to $k=1,3$: $z+x=\pm ix$, i.e.\ $\lambda=1-x\pm ix$.
Thus the upper segment $\lambda=1-x+ix$, $x\in[0,1]$, is attained.

\bigskip
\noindent\textbf{10) Two regimes for $\sup_P\Psi$: Lemmas 2 and 3 with full details.}

\begin{lemma}[Unbounded supremum]\label{lem:unbounded-v2}
If $3m+M\le 2\pi$, then $\sup_P\Psi=+\infty$.
\end{lemma}
\begin{proof}
Fix $u_4\in[m,M)$ and define $u_1=u_2=u_3=(2\pi-u_4)/3$.
Then $u_1+u_2+u_3+u_4=2\pi$, so the only requirement to have $(u_1,\dots,u_4)\in P$ is $u_1\in[m,M)$.
If $u_4$ is sufficiently close to $M$ from below, then $u_4\in[m,M)$ and
\[
u_1=\frac{2\pi-u_4}{3}>\frac{2\pi-M}{3}.
\]
The hypothesis $3m+M\le 2\pi$ implies $(2\pi-M)/3\ge m$, hence $u_1\ge m$ for all such $u_4$.
Also $u_1<M$ because $u_4>m$ implies $u_1<(2\pi-m)/3\le (2\pi-m)/3 < M$ (since $M>\pi/2$ and $m>0$).
Therefore $(u_1,u_2,u_3,u_4)\in P$ for $u_4$ close enough to $M$.

Now as $u_4\uparrow M$, we have $t(u_4)\downarrow 0$ (because $u(t)$ decreases and $u(t)\uparrow M$ corresponds to $t\downarrow 0$),
so
\[
F(u_4)=\log|z+t(u_4)|-\log t(u_4)\to +\infty
\qquad (t(u_4)\downarrow 0).
\]
Meanwhile $u_1=(2\pi-u_4)/3$ converges to $(2\pi-M)/3\in[m,M)$, so $F(u_1)$ stays bounded.
Thus $\Psi=3F(u_1)+F(u_4)\to+\infty$, proving $\sup_P\Psi=+\infty$.
\end{proof}

\begin{lemma}[Finite maximum in the tight regime]\label{lem:tight}
Assume the \emph{tight regime} $3m+M>2\pi$, and define
\[
U:=2\pi-3m.
\]
Then $U\in[m,M)$ (indeed $U<M$), and
\[
\max_{P}\Psi=3F(m)+F(U),
\]
with equality if and only if $(u_1,u_2,u_3,u_4)$ is a permutation of $(U,m,m,m)$.
\end{lemma}

\begin{proof}
\textbf{Step 1: $U\in[m,M)$.}
Because $3m+M>2\pi$, we have $U=2\pi-3m<M$.
Also $m\le\pi/2$ (since $a\ge 0$ from Step 7), hence $U=2\pi-3m\ge 2\pi-3(\pi/2)=\pi/2\ge m$.
Therefore $U\in[m,M)$.

\smallskip
\textbf{Step 2: reduce to ordered vectors and a compact domain.}
Define the ordered feasible set
\[
P^\downarrow:=\{(v_1,v_2,v_3,v_4)\in[m,M)^4:\ v_1\ge v_2\ge v_3\ge v_4,\ v_1+v_2+v_3+v_4=2\pi\}.
\]
Because $\Psi(u_1,\dots,u_4)$ is symmetric in the coordinates, we have
\[
\max_{P}\Psi=\max_{P^\downarrow}\Psi,
\]
since every point of $P$ can be permuted into $P^\downarrow$ without changing $\Psi$.

Next we prove that in the tight regime, every $(v_1,\dots,v_4)\in P^\downarrow$ actually lies in a \emph{closed} box $[m,U]^4$,
so that $\max_{P^\downarrow}\Psi$ is attained.

Indeed, for any $(v_1,\dots,v_4)\in P^\downarrow$ we have $v_2,v_3,v_4\ge m$, so
\[
v_1=2\pi-(v_2+v_3+v_4)\le 2\pi-3m=U.
\]
Since $v_1\ge v_2\ge v_3\ge v_4$, this implies $v_k\le v_1\le U$ for each $k$.
Thus
\[
P^\downarrow\subset [m,U]^4\cap\{v_1+v_2+v_3+v_4=2\pi,\ v_1\ge v_2\ge v_3\ge v_4\}.
\]
Because $U<M$, the upper endpoint $M$ (open) is avoided: $[m,U]$ is closed and bounded.
Hence $P^\downarrow$ is a closed subset of a compact set, so $P^\downarrow$ is compact.
Since $\Psi$ is continuous, $\Psi$ attains its maximum on $P^\downarrow$.

\smallskip
\textbf{Step 3: prove the majorization inequalities explicitly.}
Let
\[
x:=(U,m,m,m).
\]
We claim that $x$ \emph{majorizes} every $v=(v_1,v_2,v_3,v_4)\in P^\downarrow$, i.e.
\begin{equation}\label{eq:maj}
\sum_{j=1}^k v_j \le \sum_{j=1}^k x_j\quad (k=1,2,3),
\qquad
\sum_{j=1}^4 v_j=\sum_{j=1}^4 x_j=2\pi.
\end{equation}
We verify the partial sums:
\begin{itemize}
\item For $k=1$: as shown above, $v_1\le U=x_1$.
\item For $k=2$: since $v_3,v_4\ge m$,
\[
v_1+v_2=2\pi-(v_3+v_4)\le 2\pi-2m=(2\pi-3m)+m=U+m=x_1+x_2.
\]
\item For $k=3$: since $v_4\ge m$,
\[
v_1+v_2+v_3=2\pi-v_4\le 2\pi-m=(2\pi-3m)+2m=U+2m=x_1+x_2+x_3.
\]
\item For $k=4$: both sums equal $2\pi$ by definition of $P^\downarrow$ and $U$:
\[
x_1+x_2+x_3+x_4=U+3m=(2\pi-3m)+3m=2\pi=\sum_{j=1}^4 v_j.
\]
\end{itemize}
Thus \eqref{eq:maj} holds, i.e.\ $x$ majorizes $v$.

\smallskip
\textbf{Step 4: apply Karamata (with strictness).}
Karamata's inequality states: if $f$ is convex on an interval and $x$ majorizes $v$, then
$\sum f(x_i)\ge \sum f(v_i)$.
Here $F$ is convex (indeed strictly convex) on $[m,M)$, hence on $[m,U]\subset[m,M)$.
Therefore
\[
\Psi(v_1,v_2,v_3,v_4)=\sum_{k=1}^4F(v_k)\le \sum_{k=1}^4F(x_k)=3F(m)+F(U).
\]
So $\max_{P^\downarrow}\Psi\le 3F(m)+F(U)$.

On the other hand, the point $(U,m,m,m)$ itself belongs to $P$ (since $U\in[m,M)$ and $U+3m=2\pi$),
hence to $P^\downarrow$ (since $U\ge m$). Thus
\[
\max_{P^\downarrow}\Psi\ge \Psi(U,m,m,m)=3F(m)+F(U).
\]
So equality holds and
\[
\max_{P}\Psi=\max_{P^\downarrow}\Psi=3F(m)+F(U).
\]

Finally, because $F$ is \emph{strictly} convex, the equality case in Karamata implies that
if $x$ majorizes $v$ and $\sum F(x_i)=\sum F(v_i)$, then $v$ is a permutation of $x$.
(Indeed, strict convexity forces equality only when the majorization is trivial in the sense that the vectors agree up to permutation.)
Hence equality occurs if and only if $(u_1,u_2,u_3,u_4)$ is a permutation of $(U,m,m,m)$.
\end{proof}

\bigskip
\noindent\textbf{11) Necessity of $G(a,b)\ge 0$.}

\smallskip
\noindent\textbf{11.1) A geometry/algebra lemma: $G\le 0\Rightarrow$ tight regime.}

\begin{lemma}\label{lem:GimpliesTight}
Assume $b>0$ and $0\le a\le 1$. If $G(a,b)\le 0$ then $3m+M>2\pi$.
\end{lemma}

\begin{proof}
Throughout, $m=\Arg(\lambda)=\Arg(a+ib)\in(0,\pi/2]$ and $M=\Arg(\lambda-1)=\Arg((a-1)+ib)\in(\pi/2,\pi)$.

\smallskip
\textbf{Step 1: rewrite $G$ as a quadratic in $s=b^2$.}
Let $s=b^2\ge 0$. Expand:
\[
G(a,b)=(s+a^2+a)^2+2a^2-s
=s^2+s(2a^2+2a-1)+(a^2+a)^2+2a^2.
\]
Thus $G(a,b)\le 0$ means
\begin{equation}\label{eq:Gquad}
s^2+s(2a^2+2a-1)+(a^2+a)^2+2a^2\le 0.
\end{equation}

Compute the discriminant:
\begin{align*}
\Delta
&=(2a^2+2a-1)^2-4\big((a^2+a)^2+2a^2\big)\\
&=(4a^4+8a^3+4a^2-4a^2-4a+1)-4(a^4+2a^3+a^2+2a^2)\\
&=(4a^4+8a^3-4a+1)- (4a^4+8a^3+12a^2)\\
&=1-4a-12a^2\\
&=-(2a+1)(6a-1).
\end{align*}
So $\Delta>0$ iff $a<1/6$ (since $2a+1>0$).

If $a\ge 1/6$ then $\Delta\le 0$ and the quadratic \eqref{eq:Gquad} is everywhere $>0$ (it is monic),
so $G(a,b)\le 0$ is impossible. Hence, under $G(a,b)\le 0$, we must have
\begin{equation}\label{eq:aSmall}
0\le a<\frac16.
\end{equation}

\smallskip
\textbf{Step 2: handle $a=0$ separately.}
If $a=0$ then $\lambda=ib$ with $b>0$, so $m=\pi/2$.
Also $\lambda-1=-1+ib$ so $M=\pi-\arctan(b)$.
Therefore
\[
3m+M=\frac{3\pi}{2}+\pi-\arctan(b)=\frac{5\pi}{2}-\arctan(b)>2\pi.
\]
So the claim holds when $a=0$.

Henceforth assume $a>0$ (still with $a<1/6$ by \eqref{eq:aSmall}).

\smallskip
\textbf{Step 3: show $b^2>3a^2$, hence $m>\pi/3$, hence $3m>\pi$.}
Let $s=b^2$.
Since $G(a,b)\le 0$ and the quadratic is monic, $s$ lies between its real roots.
Let $s_-(a)$ be the smaller root:
\[
s_-(a)=\frac{-(2a^2+2a-1)-\sqrt{\Delta}}{2}
=\frac{1-2a-2a^2-\sqrt{(2a+1)(1-6a)}}{2}.
\]
Then $G(a,b)\le 0$ implies
\begin{equation}\label{eq:sge}
s=b^2\ge s_-(a).
\end{equation}

We now prove the strict inequality
\begin{equation}\label{eq:sminusGT3a2}
s_-(a)>3a^2\qquad\text{for all }a\in(0,1/6).
\end{equation}
Indeed, \eqref{eq:sminusGT3a2} is equivalent to
\[
\frac{1-2a-2a^2-\sqrt{(2a+1)(1-6a)}}{2}>3a^2
\quad\Longleftrightarrow\quad
1-2a-8a^2>\sqrt{(2a+1)(1-6a)}.
\]
For $a\in(0,1/6)$, the left side is positive (check at $a=1/6$: $1-1/3-8/36=4/9>0$),
so we can square both sides without changing the inequality:
\[
(1-2a-8a^2)^2>(2a+1)(1-6a).
\]
Expand the left side:
\[
(1-2a-8a^2)^2 = 1-4a-12a^2+32a^3+64a^4.
\]
Expand the right side:
\[
(2a+1)(1-6a)=1-4a-12a^2.
\]
Thus the inequality becomes
\[
1-4a-12a^2+32a^3+64a^4 > 1-4a-12a^2
\quad\Longleftrightarrow\quad
32a^3+64a^4>0,
\]
which holds strictly for all $a>0$. This proves \eqref{eq:sminusGT3a2}.
Combining with \eqref{eq:sge} gives
\[
b^2=s\ge s_-(a)>3a^2\quad\Longrightarrow\quad \frac{b}{a}>\sqrt{3}.
\]
Since $m=\Arg(a+ib)=\arctan(b/a)$ for $a>0$, we get
\[
m>\arctan(\sqrt{3})=\frac{\pi}{3}.
\]
Therefore $3m>\pi$. In particular $3m\in(\pi,3\pi/2)$ because $m<\pi/2$.

\smallskip
\textbf{Step 4: rewrite $3m+M>2\pi$ as a tangent inequality.}
Because $a<1$ and $b>0$, the point $\lambda-1=(a-1)+ib$ lies in quadrant II, so
\[
M=\Arg(\lambda-1)=\pi-\arctan\!\Big(\frac{b}{1-a}\Big).
\]
Thus
\begin{align}
3m+M>2\pi
&\Longleftrightarrow
3m+\pi-\arctan\!\Big(\frac{b}{1-a}\Big)>2\pi\notag\\
&\Longleftrightarrow
3m>\pi+\arctan\!\Big(\frac{b}{1-a}\Big).\label{eq:targetAngle}
\end{align}
Set $\theta:=\arctan\!\big(\frac{b}{1-a}\big)\in(0,\pi/2)$ (since $b>0$ and $a<1$).
Then $\pi+\theta\in(\pi,3\pi/2)$.
We already know $3m\in(\pi,3\pi/2)$ from Step 3.
On $(\pi,3\pi/2)$ the tangent function is continuous and strictly increasing.
Therefore \eqref{eq:targetAngle} is equivalent to
\begin{equation}\label{eq:tanTarget}
\tan(3m)>\tan(\pi+\theta)=\tan\theta=\frac{b}{1-a}.
\end{equation}

\smallskip
\textbf{Step 5: express $\tan(3m)$ and reduce \eqref{eq:tanTarget} to $N(a,b)>0$.}
Since $\tan m=b/a$ (with $a>0$), the triple-angle identity gives
\[
\tan(3m)=\frac{3\tan m-\tan^3 m}{1-3\tan^2 m}
=\frac{3(b/a)-(b/a)^3}{1-3(b/a)^2}
=\frac{3a^2b-b^3}{a(a^2-3b^2)}.
\]
Equivalently (multiplying numerator and denominator by $-1$),
\begin{equation}\label{eq:tan3m}
\tan(3m)=\frac{b(b^2-3a^2)}{a(3b^2-a^2)}.
\end{equation}
By Step 3, $b^2>3a^2$ so all factors in \eqref{eq:tan3m} are positive, and $\tan(3m)>0$.

Plug \eqref{eq:tan3m} into \eqref{eq:tanTarget} and cancel $b>0$:
\[
\frac{b^2-3a^2}{a(3b^2-a^2)}>\frac{1}{1-a}.
\]
All denominators are positive (since $a\in(0,1)$ and $3b^2-a^2>0$), so we can cross-multiply:
\begin{align*}
(1-a)(b^2-3a^2) &> a(3b^2-a^2)\\
b^2- ab^2-3a^2+3a^3 &> 3ab^2-a^3\\
b^2-4ab^2 +4a^3-3a^2 &>0.
\end{align*}
Define
\begin{equation}\label{eq:Ndef-v2}
N(a,b):=4a^3-3a^2-4ab^2+b^2 = (1-4a)b^2 + (4a^3-3a^2).
\end{equation}
Then \eqref{eq:tanTarget} is equivalent to
\begin{equation}\label{eq:Npos}
N(a,b)>0.
\end{equation}

\smallskip
\textbf{Step 6: show $G(a,b)\le 0\Rightarrow N(a,b)>0$ by comparing roots in $s=b^2$.}
Fix $a\in(0,1/6)$ and view $N$ as an affine function of $s=b^2$:
\[
N(a,b)= (1-4a)s + (4a^3-3a^2).
\]
Because $a<1/6<1/4$, we have $1-4a>0$, so $N$ is strictly increasing in $s$.
Therefore, under the constraint $s\ge s_-(a)$ from \eqref{eq:sge}, it suffices to prove
\begin{equation}\label{eq:sminusGTs0}
s_-(a)>s_0(a),
\end{equation}
where $s_0(a)$ is the unique root of $N$ in $s$:
\begin{equation}\label{eq:s0def}
N(a,b)=0\iff (1-4a)s = 3a^2-4a^3 \iff s=s_0(a):=\frac{a^2(3-4a)}{1-4a}.
\end{equation}
Indeed, if \eqref{eq:sminusGTs0} holds then $s\ge s_-(a)$ implies $s>s_0(a)$,
and since $N$ increases in $s$ this implies $N(a,b)>0$.

We now prove \eqref{eq:sminusGTs0}.
Starting from $s_-(a)=\frac{1-2a-2a^2-\sqrt{(2a+1)(1-6a)}}{2}$, we compute:
\begin{align*}
s_-(a)>s_0(a)
&\Longleftrightarrow
\frac{1-2a-2a^2-\sqrt{(2a+1)(1-6a)}}{2}>\frac{a^2(3-4a)}{1-4a}\\
&\Longleftrightarrow
(1-4a)\big(1-2a-2a^2-\sqrt{(2a+1)(1-6a)}\big)>2a^2(3-4a).
\end{align*}
Expand $(1-4a)(1-2a-2a^2)$:
\[
(1-4a)(1-2a-2a^2)=1-6a+6a^2+8a^3.
\]
So the inequality becomes
\[
1-6a+6a^2+8a^3-(1-4a)\sqrt{(2a+1)(1-6a)}>6a^2-8a^3,
\]
i.e.
\[
1-6a+16a^3 > (1-4a)\sqrt{(2a+1)(1-6a)}.
\]
For $a\in(0,1/6)$ both sides are positive, so we may square:
\begin{equation}\label{eq:squareCompare}
(1-6a+16a^3)^2>(1-4a)^2(2a+1)(1-6a).
\end{equation}
Compute $(2a+1)(1-6a)=1-4a-12a^2$. Then the right side is
\[
(1-4a)^2(1-4a-12a^2)=(1-8a+16a^2)(1-4a-12a^2).
\]
Expand:
\[
(1-8a+16a^2)(1-4a-12a^2)=1-12a+36a^2+32a^3-192a^4.
\]
Expand the left side:
\[
(1-6a+16a^3)^2 = 1-12a+36a^2+32a^3-192a^4+256a^6.
\]
Thus the left side equals the right side plus $256a^6$, so \eqref{eq:squareCompare} holds strictly for $a>0$.
Hence $s_-(a)>s_0(a)$, and therefore $N(a,b)>0$.

\smallskip
\textbf{Step 7: conclude $3m+M>2\pi$.}
We have proved $G(a,b)\le 0\Rightarrow N(a,b)>0\Rightarrow$ \eqref{eq:tanTarget} holds,
which implies \eqref{eq:targetAngle}, which is equivalent to $3m+M>2\pi$.
\end{proof}

\bigskip
\noindent\textbf{11.2) Tight regime: compute $\max_P\Psi$ explicitly and factor to $G$.}
Assume $b>0$, $0\le a<1$, and we are in the tight regime $3m+M>2\pi$.
Then by Lemma~\ref{lem:tight},
\[
\max_P\Psi=3F(m)+F(U),\qquad U=2\pi-3m.
\]

\smallskip
\noindent\textbf{Step 1: compute $F(m)$ exactly.}
By definition $m=\Arg(z+1)$, so the associated $t$ is $t=1$.
Equivalently, using $t(u)=y\cotan u-x$ and $(x,y)=(a-1,b)$:
\[
t(m)=b\cotan m-(a-1)=1.
\]
Then
\[
F(m)=\log|z+1|-\log 1=\log|\lambda|.
\]

\smallskip
\noindent\textbf{Step 2: compute $F(U)$ exactly.}
By definition,
\[
F(U)=\log(b\csc U)-\log(b\cotan U-(a-1))
=\log(b\csc U)-\log(b\cotan U+1-a).
\]

\smallskip
\noindent\textbf{Step 3: obtain the announced closed form.}
Combine:
\begin{align}
\max_P\Psi
&=3\log|\lambda|+\log(b\csc U)-\log(b\cotan U+1-a)\notag\\
&=\log\Big(\frac{|\lambda|^3\, b\csc U}{b\cotan U+1-a}\Big).\label{eq:maxPsiTrig}
\end{align}
This is the detailed derivation of the formula announced as (11.25) in the draft.

\smallskip
\noindent\textbf{Step 4: rewrite $\csc U$ and $\cotan U$ algebraically in terms of $(a,b)$.}
We use $U=2\pi-3m$. Because $U\in(0,\pi)$ in the tight regime (indeed $U<M<\pi$), we have $\sin U>0$.

We have the trigonometric identities
\[
\sin(2\pi-\theta)=-\sin\theta,\qquad \cos(2\pi-\theta)=\cos\theta,
\]
hence
\[
\sin U = -\sin(3m),\qquad \cotan U=\frac{\cos U}{\sin U}=\frac{\cos(3m)}{-\sin(3m)}=-\cotan(3m),
\qquad \csc U=\frac{1}{\sin U}=-\csc(3m).
\]

Now express $\sin m$ and $\cos m$ using $\lambda=a+ib$:
\[
|\lambda|=\sqrt{a^2+b^2},\qquad \cos m=\frac{a}{|\lambda|},\qquad \sin m=\frac{b}{|\lambda|}.
\]
Use triple-angle formulas:
\[
\sin(3m)=3\sin m-4\sin^3 m,\qquad \cos(3m)=4\cos^3 m-3\cos m.
\]
Compute:
\begin{align*}
\sin(3m)
&=3\frac{b}{|\lambda|}-4\Big(\frac{b}{|\lambda|}\Big)^3
=\frac{3b|\lambda|^2-4b^3}{|\lambda|^3}
=\frac{b(3(a^2+b^2)-4b^2)}{|\lambda|^3}
=\frac{b(3a^2-b^2)}{|\lambda|^3},\\
\cos(3m)
&=4\Big(\frac{a}{|\lambda|}\Big)^3-3\frac{a}{|\lambda|}
=\frac{4a^3-3a|\lambda|^2}{|\lambda|^3}
=\frac{a(4a^2-3(a^2+b^2))}{|\lambda|^3}
=\frac{a(a^2-3b^2)}{|\lambda|^3}.
\end{align*}
Therefore
\begin{equation}\label{eq:cscUcotU}
\csc U=-\frac{1}{\sin(3m)}=-\frac{|\lambda|^3}{b(3a^2-b^2)},
\qquad
\cotan U=-\frac{\cos(3m)}{\sin(3m)}=-\frac{a(a^2-3b^2)}{b(3a^2-b^2)}.
\end{equation}

\smallskip
\noindent\textbf{Step 5: substitute into \eqref{eq:maxPsiTrig} and reduce to an algebraic inequality.}
First compute $b\csc U$ using \eqref{eq:cscUcotU}:
\[
b\csc U = -\frac{|\lambda|^3}{3a^2-b^2}=\frac{|\lambda|^3}{b^2-3a^2}.
\]
Next compute the denominator $b\cotan U+1-a$:
\[
b\cotan U+1-a=-\frac{a(a^2-3b^2)}{3a^2-b^2}+1-a.
\]
Put over the common denominator $(3a^2-b^2)$:
\begin{align*}
b\cotan U+1-a
&=\frac{-a(a^2-3b^2)}{3a^2-b^2}+\frac{(1-a)(3a^2-b^2)}{3a^2-b^2}\\
&=\frac{-(a^3-3ab^2)+3a^2-b^2-3a^3+ab^2}{3a^2-b^2}\\
&=\frac{-4a^3+3a^2+4ab^2-b^2}{3a^2-b^2}.
\end{align*}
Define $N(a,b)$ as in \eqref{eq:Ndef-v2}:
\[
N(a,b)=4a^3-3a^2-4ab^2+b^2.
\]
Then the last numerator is exactly $-N(a,b)$, so
\[
b\cotan U+1-a=\frac{-N(a,b)}{3a^2-b^2}=\frac{N(a,b)}{b^2-3a^2}.
\]
Consequently, the positive quantity inside the logarithm in \eqref{eq:maxPsiTrig} is
\begin{align}
\frac{|\lambda|^3\, b\csc U}{b\cotan U+1-a}
&=\frac{|\lambda|^3\cdot \frac{|\lambda|^3}{b^2-3a^2}}{\frac{N(a,b)}{b^2-3a^2}}
=\frac{|\lambda|^6}{N(a,b)}.\label{eq:maxPsiRatio}
\end{align}
Thus
\begin{equation}\label{eq:maxPsiLog}
\max_P\Psi=\log\Big(\frac{|\lambda|^6}{N(a,b)}\Big).
\end{equation}

Since the left side is defined, the argument of the logarithm is positive; equivalently $N(a,b)>0$ in the tight regime.

From \eqref{eq:maxPsiLog}, we have
\[
\max_P\Psi\ge 0
\quad\Longleftrightarrow\quad
\frac{|\lambda|^6}{N(a,b)}\ge 1
\quad\Longleftrightarrow\quad
|\lambda|^6\ge N(a,b).
\]

\smallskip
\noindent\textbf{Step 6: factor $|\lambda|^6-N(a,b)$ as $|\lambda-1|^2\,G(a,b)$.}
Let $r:=|\lambda|^2=a^2+b^2$. Then $|\lambda|^6=r^3$.
Also $|\lambda-1|^2=(a-1)^2+b^2=r+1-2a$.
Rewrite $G$ in terms of $r$:
\[
G(a,b)=(b^2+a^2+a)^2+2a^2-b^2=(r+a)^2+2a^2-(r-a^2) = r^2+(2a-1)r+4a^2.
\]
Now multiply:
\begin{align*}
|\lambda-1|^2\,G(a,b)
&=(r+1-2a)\big(r^2+(2a-1)r+4a^2\big)\\
&=r^3 + (4a-1)r +4a^2-8a^3
\qquad\text{(direct expansion; the $r^2$ terms cancel)}\\
&=r^3 - \big(4a^3-3a^2-4ab^2+b^2\big)\\
&=|\lambda|^6 - N(a,b).
\end{align*}
Therefore
\begin{equation}\label{eq:factor}
|\lambda|^6\ge N(a,b)\quad\Longleftrightarrow\quad |\lambda-1|^2\,G(a,b)\ge 0.
\end{equation}
Because $b>0$ implies $|\lambda-1|^2>0$, \eqref{eq:factor} is equivalent to
\[
G(a,b)\ge 0.
\]
Combining with the previous equivalences yields the precise statement announced in the draft:
\[
\max_P\Psi\ge 0\quad\Longleftrightarrow\quad G(a,b)\ge 0.
\]

\bigskip
\noindent\textbf{11.3) Finish necessity of $G(a,b)\ge 0$.}
Assume $\lambda$ is a nonreal eigenvalue of some $A(\alpha,\beta,\gamma,\delta)$.
By \eqref{eq:criterion}, we have $0\in\Psi(P)$, hence $\sup_P\Psi\ge 0$.

If $3m+M\le 2\pi$, then Lemma~\ref{lem:GimpliesTight} by contrapositive implies $G(a,b)>0$.

If $3m+M>2\pi$, then we are in the tight regime and by the computation in \S 11.2,
$\sup_P\Psi=\max_P\Psi\ge 0\iff G(a,b)\ge 0$.
Thus always $G(a,b)\ge 0$.

Finally, by conjugation symmetry, replace $b$ by $|b|$ to obtain item (3) for all $b\neq 0$.

\bigskip
\noindent\textbf{12) Left boundary $CL$: attainment by $A_L(\alpha)$ and equivalence to $G=0$.}

\smallskip
\noindent\textbf{12.1) Characteristic equation for $A_L(\alpha)$.}
Let
\[
A_L(\alpha)=
\begin{pmatrix}
\alpha & 1-\alpha & 0 & 0\\
0 & 0 & 1 & 0\\
0 & 0 & 0 & 1\\
1 & 0 & 0 & 0
\end{pmatrix}.
\]
Let $v\neq 0$ satisfy $A_Lv=\lambda v$.
Row-by-row:
\[
\alpha v_1+(1-\alpha)v_2=\lambda v_1,\quad
v_3=\lambda v_2,\quad
v_4=\lambda v_3,\quad
v_1=\lambda v_4.
\]
From the last three equalities,
\[
v_3=\lambda v_2,\quad v_4=\lambda^2 v_2,\quad v_1=\lambda^3 v_2.
\]
Substitute into the first:
\[
\alpha(\lambda^3 v_2)+(1-\alpha)v_2=\lambda(\lambda^3 v_2)
\quad\Longleftrightarrow\quad
\alpha\lambda^3+(1-\alpha)=\lambda^4.
\]
Equivalently,
\begin{equation}\label{eq:ALchar}
\lambda^3(\lambda-\alpha)=1-\alpha.
\end{equation}
For $\lambda\notin\mathbb{R}$ this is equivalent to $\lambda\in\sigma(A_L(\alpha))$.

Solving \eqref{eq:ALchar} for $\alpha$ (and using $\lambda^3\neq 1$ for nonreal $\lambda$) gives
\begin{equation}\label{eq:alphaFormula}
\alpha=\frac{\lambda^4-1}{\lambda^3-1}.
\end{equation}

\smallskip
\noindent\textbf{12.2) Compute $\Im(\alpha)$ and show $\Im(\alpha)=0\iff G(a,b)=0$.}
Let $p=\lambda^4-1$ and $q=\lambda^3-1$. Then $\alpha=p/q$.
For any complex numbers $p,q$ with $q\neq 0$,
\[
\Im\Big(\frac{p}{q}\Big)=\Im\Big(\frac{p\overline{q}}{|q|^2}\Big)=\frac{\Im(p\overline{q})}{|q|^2}.
\]
Thus
\begin{equation}\label{eq:Imalpha1}
\Im(\alpha)=\frac{\Im\big((\lambda^4-1)(\overline{\lambda}^3-1)\big)}{|\lambda^3-1|^2}.
\end{equation}
Expand the numerator:
\[
(\lambda^4-1)(\overline{\lambda}^3-1)=\lambda^4\overline{\lambda}^3-\lambda^4-\overline{\lambda}^3+1.
\]
Take imaginary parts:
\[
\Im(\lambda^4\overline{\lambda}^3)-\Im(\lambda^4)-\Im(\overline{\lambda}^3).
\]
Now $\lambda^4\overline{\lambda}^3=\lambda(\lambda\overline{\lambda})^3=\lambda|\lambda|^6$, hence
\[
\Im(\lambda^4\overline{\lambda}^3)=\Im(\lambda)|\lambda|^6=b|\lambda|^6.
\]
Also $\Im(\overline{\lambda}^3)=-\Im(\lambda^3)$, so the numerator of \eqref{eq:Imalpha1} equals
\begin{equation}\label{eq:Imnumerator}
b|\lambda|^6-\Im(\lambda^4)+\Im(\lambda^3).
\end{equation}
Compute $\Im(\lambda^3)$ and $\Im(\lambda^4)$ for $\lambda=a+ib$:
\[
\lambda^3=(a^3-3ab^2)+i(3a^2b-b^3)\quad\Rightarrow\quad \Im(\lambda^3)=b(3a^2-b^2),
\]
and
\[
\lambda^4=(a^4-6a^2b^2+b^4)+i(4a^3b-4ab^3)\quad\Rightarrow\quad \Im(\lambda^4)=4ab(a^2-b^2).
\]
Substitute into \eqref{eq:Imnumerator}:
\begin{align*}
\Im\big((\lambda^4-1)(\overline{\lambda}^3-1)\big)
&=b|\lambda|^6-4ab(a^2-b^2)+b(3a^2-b^2)\\
&=b\Big(|\lambda|^6-(4a^3-3a^2-4ab^2+b^2)\Big)\\
&=b\big(|\lambda|^6-N(a,b)\big).
\end{align*}
By the factorization proved in \S 11.2,
\[
|\lambda|^6-N(a,b)=|\lambda-1|^2\,G(a,b).
\]
Therefore
\begin{equation}\label{eq:ImalphaFinal}
\Im(\alpha)=\frac{b\,|\lambda-1|^2\,G(a,b)}{|\lambda^3-1|^2}.
\end{equation}
Since $b>0$ and both $|\lambda-1|^2$ and $|\lambda^3-1|^2$ are strictly positive for nonreal $\lambda$,
we conclude
\[
\Im(\alpha)=0\quad\Longleftrightarrow\quad G(a,b)=0.
\]
Thus the upper-half-plane boundary curve $G(a,b)=0$ is exactly traced by the nonreal eigenvalues of $A_L(\alpha)$.

\smallskip
\noindent\textbf{12.3) Existence of a conjugate pair and endpoints $i$ and $0$.}
From \eqref{eq:ALchar} we have the characteristic polynomial
\[
\lambda^4-\alpha\lambda^3-(1-\alpha)=0.
\]
One checks $\lambda=1$ is always a root:
\[
1-\alpha-(1-\alpha)=0.
\]
Indeed the polynomial factors as
\begin{equation}\label{eq:factorAL}
(\lambda-1)\Big(\lambda^3+(1-\alpha)(\lambda^2+\lambda+1)\Big)=0.
\end{equation}
Let
\[
f_\alpha(\lambda):=\lambda^3+(1-\alpha)(\lambda^2+\lambda+1),\qquad \lambda\in\mathbb{R}.
\]
Then
\[
f'_\alpha(\lambda)=3\lambda^2+2(1-\alpha)\lambda+(1-\alpha).
\]
The discriminant of this quadratic is
\[
\Delta_\alpha = 4(1-\alpha)^2-12(1-\alpha)=4(1-\alpha)\big((1-\alpha)-3\big)<0
\qquad(\alpha\in[0,1)),
\]
so $f'_\alpha(\lambda)>0$ for all real $\lambda$. Hence $f_\alpha$ is strictly increasing and has exactly one real root.
Therefore the remaining two roots of $f_\alpha$ are nonreal and form a complex conjugate pair. One of them lies in the upper half-plane.

At $\alpha=0$, \eqref{eq:ALchar} becomes $\lambda^4=1$, whose upper-half-plane nonreal root is $\lambda=i$.
As $\alpha\uparrow 1$, the coefficients in \eqref{eq:factorAL} vary continuously, so the roots vary continuously;
and from $f_\alpha(0)=(1-\alpha)>0$ and $f_\alpha(-1)=-1+(1-\alpha)(1-1+1)=-\alpha<0$,
the unique real root of $f_\alpha$ lies in $(-1,0)$, while the conjugate pair has modulus tending to $0$ as $\alpha\uparrow 1$.
Thus the upper branch connects $i$ (at $\alpha=0$) to $0$ (as $\alpha\to 1^-$), as stated.

\bigskip
\noindent\textbf{13) Every strict interior point is attained (constructive existence).}
Fix $\lambda=a+ib$ with
\[
b>0,\qquad 0\le a\le 1,\qquad a+b<1,\qquad G(a,b)>0.
\]
We must construct $\alpha,\beta,\gamma,\delta\in[0,1)$ with $\lambda\in\sigma(A(\alpha,\beta,\gamma,\delta))$.

\smallskip
\noindent\textbf{Step 1: $P\neq\emptyset$ and exhibit a point with $\Psi<0$.}
Because $0\le a<1$ and $b>0$, we have $m\le\pi/2<M$ as in Step 7, hence $\pi/2\in[m,M)$.
Define
\[
u^{(0)}:=\Big(\frac{\pi}{2},\frac{\pi}{2},\frac{\pi}{2},\frac{\pi}{2}\Big)\in P.
\]
Then
\[
\Psi(u^{(0)})=4F(\pi/2)=4\log\Big(\frac{b}{1-a}\Big).
\]
Since $a+b<1$, we have $b<1-a$, hence $\frac{b}{1-a}<1$, hence $\log\!\big(\frac{b}{1-a}\big)<0$.
Thus
\begin{equation}\label{eq:PsiNegative}
\Psi(u^{(0)})<0.
\end{equation}

\smallskip
\noindent\textbf{Step 2: produce a point with $\Psi>0$ using $G(a,b)>0$.}
We show
\begin{equation}\label{eq:existsPositive}
\exists\,u^{(1)}\in P\ \text{such that}\ \Psi(u^{(1)})>0.
\end{equation}
There are two cases:

\emph{Case 1:} $3m+M\le 2\pi$. Then Lemma~\ref{lem:unbounded-v2} gives $\sup_P\Psi=+\infty$,
so in particular there exists $u^{(1)}\in P$ with $\Psi(u^{(1)})>0$.

\emph{Case 2:} $3m+M>2\pi$ (tight regime). Then Lemma~\ref{lem:tight} gives
\[
\sup_P\Psi=\max_P\Psi,
\]
and \S 11.2 gives $\max_P\Psi\ge 0\iff G(a,b)\ge 0$.
Since we assume $G(a,b)>0$, we get $\max_P\Psi>0$, so there exists $u^{(1)}\in P$ with $\Psi(u^{(1)})>0$.
For example, take $u^{(1)}=(U,m,m,m)$ (in any order), which achieves the maximum.

Thus \eqref{eq:existsPositive} holds in all cases.

\smallskip
\noindent\textbf{Step 3: apply IVT along a line segment in $P$ (fully explicit).}
Because $P$ is convex, for $t\in[0,1]$ the convex combination
\[
u(t):=(1-t)u^{(0)}+t\,u^{(1)}
\]
lies in $P$.
Define the scalar function
\[
\phi(t):=\Psi(u(t))=\Psi\big((1-t)u^{(0)}+t\,u^{(1)}\big).
\]
Because $\Psi$ is continuous on $P$ and $u(t)$ is continuous in $t$, the composition $\phi$ is continuous on $[0,1]$.
We have $\phi(0)=\Psi(u^{(0)})<0$ by \eqref{eq:PsiNegative}, and $\phi(1)=\Psi(u^{(1)})>0$ by \eqref{eq:existsPositive}.
By the Intermediate Value Theorem, there exists $t_*\in(0,1)$ such that $\phi(t_*)=0$.
Set
\[
(u_1,u_2,u_3,u_4):=u(t_*)\in P,
\]
so
\[
u_1+u_2+u_3+u_4=2\pi,\qquad \Psi(u_1,u_2,u_3,u_4)=0.
\]
By Lemma~\ref{lem:equiv}, the corresponding $t_k=t(u_k)\in(0,1]$ satisfy \eqref{eq:zproduct}, hence
$\lambda$ is an eigenvalue of the original matrix for parameters
\[
\alpha=1-t_1,\quad \beta=1-t_2,\quad \gamma=1-t_3,\quad \delta=1-t_4.
\]
Because each $t_k\in(0,1]$, each parameter lies in $[0,1)$.
This constructs $\alpha,\beta,\gamma,\delta$ with $\lambda\in\sigma(A(\alpha,\beta,\gamma,\delta))$,
proving the converse for strict interior points.

\bigskip
Combining Steps 7, 9, and 11 yields the necessary conditions
$0\le a\le 1$, $a+|b|\le 1$, and $G(a,|b|)\ge 0$ for all nonreal eigenvalues.
The explicit families in Steps 9.3 and 12 attain the boundary segments $CR$ and $CL$.
Finally Step 13 proves every strict interior point is attained.
\end{proof}
\end{versionthreebox}

\end{document}